\documentclass[11pt]{article}

\usepackage[margin=1in]{geometry}
\usepackage[utf8]{inputenc}
\usepackage[T1]{fontenc}
\usepackage{hyperref}
\usepackage{url}
\usepackage{booktabs}
\usepackage{amsfonts}
\usepackage{amsmath}
\usepackage{amssymb}
\usepackage{mathtools}
\usepackage{amsthm}
\usepackage{nicefrac}
\usepackage{microtype}
\usepackage{graphicx}
\usepackage[table]{xcolor}
\usepackage{subcaption}
\usepackage{algorithm}
\usepackage{algorithmic}
\usepackage{natbib}
\usepackage{dsfont}
\usepackage{placeins}

\hypersetup{
    colorlinks=true,
    linkcolor=blue,
    filecolor=blue,      
    urlcolor=blue,
    citecolor=blue
}

\newcommand{\mbf}[1]{\mathbf{#1}}

\newcommand{\defeq}{:=}

\newcommand{\indic}[1]{\mbf{1}\left\{#1\right\}}

\theoremstyle{plain}
\newtheorem{theorem}{Theorem}

\newtheorem{lemma}[theorem]{Lemma}

\theoremstyle{definition}

\newtheorem{assumption}[theorem]{Assumption}

\theoremstyle{remark}

\definecolor{cornflowerblue}{RGB}{100, 149, 237}

\title{\Large \bf Top-$k$ Feature Importance Ranking}

\author{
Yuxi Chen\\
Carnegie Mellon University\\
\texttt{ericc3@andrew.cmu.edu}
\and
Tiffany Tang\\
University of Notre Dame\\
\texttt{ttang4@nd.edu}
\and
Genevera Allen\\
Columbia University\\
\texttt{genevera.allen@columbia.edu}
}

\date{}

\begin{document}

\maketitle

\begin{abstract}
Accurate ranking of important features is a fundamental challenge in interpretable machine learning with critical applications in scientific discovery and decision-making. Unlike feature selection and feature importance, the specific problem of ranking important features has received considerably less attention. We introduce RAMPART (Ranked Attributions with MiniPatches And Recursive Trimming), a framework that utilizes any existing feature importance measure in a novel algorithm specifically tailored for ranking the top-$k$ features. Our approach combines an adaptive sequential halving strategy that progressively focuses computational resources on promising features with an efficient ensembling technique using both observation and feature subsampling. Unlike existing methods that convert importance scores to ranks as post-processing, our framework explicitly optimizes for ranking accuracy. We provide theoretical guarantees showing that RAMPART achieves the correct top-$k$ ranking with high probability under mild conditions, and demonstrate through extensive simulation studies that RAMPART consistently outperforms popular feature importance methods, concluding with a high-dimensional genomics case study.
\end{abstract}

\section{Introduction}
\label{section: intro}

A key challenge in interpretable machine learning is determining not just which features influence model predictions but their relative importance ranking. In particular, accurately ranking the top-$k$ most important features would fundamentally change the decision-making and scientific discovery process in numerous high-stakes applications \citep{bhatt_explainable_2020, jaxa-rozen_sources_2021}. For example, in genomics, genome-wide association studies (GWAS) \citep{visscher201710} are by far the most common approach to identify important genes or genetic variants that are associated with disease risk. 
These data-driven studies often identify hundreds of important genetic variants.
However, to translate these findings into tangible therapeutic targets and clinical practice, wet-lab validation is necessary, but typically limited to a few dozen genetic variants, if not fewer due to its high cost \citep{fu_gene_2020, wang_epistasis_2023, pashaei_biomarker_2025}.
More generally, given resource constraints, the need to rank or prioritize a small number of top-ranked candidates for costly downstream decision-making is a common theme among many scientific and clinical pipelines.

Although feature importances have been extensively studied in machine learning, methods specifically designed for ranking the top-$k$ most important features remain under-developed. Current approaches for top-$k$ feature ranking typically rely on the heuristic of first estimating feature importance values for all features, sorting them, and then subsetting to the top-$k$ features with the largest importance \citep{lundberg_unified_2017, neuhof_confident_2024, goldwasser_statistical_2025}. However, the first step of this paradigm is particularly limiting as valuable data and computational resources are being used to estimate the importances of \textit{all} features, including those that are irrelevant or far outside of the top-$k$ that are of primary interest.
This issue is further exacerbated in realistic settings with correlated and high-dimensional data (e.g., in genomics), where existing feature importance estimates are known to be highly unstable and unreliable \citep{nicodemus2009predictor, nicodemus2011stability, hooker2021unrestricted}. These challenges highlight the need for a new paradigm to accurately rank the top-$k$ most important features.

\subsection{Our Contributions}

We address the challenges in ranking the top-$k$ features with the highest global feature importances and make several key contributions. First, we introduce RAMP (Ranked Attributions with MiniPatches), an efficient ensembling strategy that aggregates models trained on random subsamples (or ``minipatches'') of both observations and features. This approach breaks harmful correlation patterns among features while maintaining statistical power \citep{gan_model-agnostic_2022}. Building on RAMP, we then develop RAMPART (RAMP And Recursive Trimming), a model-agnostic framework for top-$k$ feature importance ranking applicable to any feature attribution method. RAMPART's novel recursive trimming approach progressively focuses computational resources on promising features while eliminating suboptimal ones—becoming increasingly precise in distinguishing between similarly ranked features as the candidate pool shrinks. Unlike existing approaches that allocate equal resources to all features, this adaptive strategy proves particularly effective in high-dimensional settings. Finally, we provide theoretical guarantees on recovering the correct top-$k$ feature importance ranking under mild assumptions, establishing explicit sample complexity bounds that may be of independent interest.  

\subsection{Related Works}
\label{subsection: related works}

\paragraph{Feature Importance} 
Although not directly designed for feature importance ranking, numerous model-specific and model-agnostic feature importance measures have been developed to quantify the contribution of each predictor feature on the model's predictions and performance \citep{molnar_interpretable_2022}. Popular model-specific approaches include regression coefficients for linear models, Mean Decrease in Impurity for tree-based methods \citep{breiman2001random}, and neural network attributions like DeepLift and Integrated Gradients \citep{shrikumar_learning_2019, sundararajan_axiomatic_2017}. Model-agnostic methods include occlusion-based \citep{lei_distribution-free_2014}, permutation-based \citep{breiman2001random}, and Shapley-based techniques \citep{lundberg_unified_2017, lundberg_local_2020}. In Section~\ref{section: empirical studies}, we will demonstrate the shortcomings of simply ranking these feature importances to obtain the top-$k$.

\paragraph{Ranking from Pairwise Comparisons}
On the other hand, viewing this problem from the lens of the ranking literature, many previous works have directly estimated rankings from pairwise comparisons. This literature includes tournament methods \citep{mohajer_active_2017}, spectral techniques \citep{negahban_rank_2017, chen_spectral_2015, chen_spectral_2019}, adaptive selection paradigms \citep{heckel_active_2016, heckel_approximate_2018}, and weighting strategies \citep{shah_simple_2018, wauthier_efficient_2013, ammar_efficient_2012}. Despite their theoretical appeal, these methods struggle to capture multivariate feature dependencies and face computational barriers in high dimensions, limiting their applicability to top-$k$ feature importance ranking.

\paragraph{Feature Importance Ranking}
More recently, several works have begun to focus specifically on feature importance ranking. \citet{kariyappa_shapkefficient_2023} developed sampling algorithms to identify top-$k$ features by Shapley values without addressing their ordering. \citet{teneggi_testing_2024} employed statistical independence testing with betting principles, primarily for semantic concept validation in vision rather than tabular data. \citet{neuhof_confident_2024} introduced a framework for quantifying uncertainty of feature importance rankings through simultaneous confidence intervals. Their approach focuses primarily on post-hoc interpretation of pre-computed importance scores rather than providing an efficient algorithmic framework for large-scale feature ranking. Most relevant to our work, \citet{goldwasser_statistical_2025} developed a sequential pairwise hypothesis testing framework for assessing the statistical significance of the top-$k$ most important features using resampled attribution scores. This approach, however, requires normality and independence assumptions that are rarely satisfied in practice and violated by correlated estimators. Computational overhead from repeated pairwise testing also limits scalability to high dimensions. Their subsequent rank verification method \citep{goldwasser_gaussian_2025} similarly assumes Gaussian distributions, constraining applicability to real-world data with non-Gaussian distributions and complex dependencies.

\paragraph{Best Arm Identification}
To avoid the current limitations of existing feature importance ranking approaches, we introduce a recursive trimming strategy, which draws inspiration from multi-armed bandits research on best arm identification, including UCB approaches \citep{audibert_best_2010, chen_nearly_2017}, Thompson Sampling \citep{russo_simple_2020}, and halving algorithms \citep{zhao_revisiting_2023}. Particularly relevant is research on best-\textit{k}-arm identification \citep{chen_efficient_2008, gao_note_2015, you_information-directed_2023}, with \citet{liu_variable_2023} successfully applying Thompson Sampling to variable selection. However, direct application of bandit algorithms to feature ranking faces key challenges: (1) assumed arm independence—violated by correlated features and (2) unknown distributions of importance measures. Our work addresses these limitations with a novel approach that efficiently produces statistically robust feature rankings while accounting for feature interdependencies.

\section{Adaptive Feature Importance Ranking}
\label{section: adaptive feature ranking}

\subsection{Problem Setup}

\label{subsection: problem setup}
Suppose we observe a dataset $\mathcal{D} = \{(\mathbf{x}_1, y_1)\, \dots, (\mathbf{x}_N, y_N)\}$ where $\mathbf{x}_i \in \mathbb{R}^M$ and $y_i \in \mathbb{R}$ denote the features and response respectively.  We assume observations are independent and identically distributed draws from an unknown joint distribution. We also assume that each feature $j \in [M] := \{1, \dots, M\}$ possesses an inherent global feature importance $\phi_j$. There are many existing metrics that can be used to quantify $\phi_j$: see \citet{lundberg_unified_2017,  molnar_interpretable_2022, fisher_all_2019} for instance. We note that the feature importance $\phi_j$ depends on the specific predictive model and importance metric used, with each method potentially defining a distinct ground truth. We define the rank of the $j$-th feature as $r_j \defeq \sum_{i = 1}^{M} \indic{|\phi_j| < |\phi_i|}$ and the $j$-th best feature as $\tau_j$. Additionally, for some pre-specified $k\ll M$, we assume that the top-$k$ features do not contain any ties: $r_{\tau_j} \neq r_{\tau_{j'}}$ for any $j, j' \in [k]$ where $j \neq j'$.
Our goal is to correctly estimate these top $k$ ranks such that $\hat{r}_{\tau_j} = r_{\tau_j}$ for $j \in [k]$.

\subsection{Motivation}
\label{subsection: motivation}

As mentioned previously, one naive approach to feature importance ranking is to first estimate importance scores $\{\hat{\phi}_{j}\}_{j = 1}^{M}$ and then sort the features accordingly. However, in realistic, high-dimensional settings, this approach faces several fundamental challenges, which motivate our proposed framework.

First, obtaining accurate feature importance estimates in high dimensions is problematic. These estimates often suffer from bias due to inherent correlations in high-dimensional data, with well-known metrics performing poorly in correlated settings \citep{chamma_variable_2023}. The estimates also exhibit high variance and unreliability, as even classical methods such as ordinary least squares coefficients are known to become unstable when the number of features approaches or exceeds the sample size \citep{hastie2009elements}. The computational burden compounds these issues: computing importances for a large number of features can be expensive, forcing popular approaches to resort to approximations. For instance, methods like Shapley values are computationally infeasible without substantial approximations that compromise their theoretical guarantees \citep{mitchell_sampling_2022, ghorbani_data_2019}.

Second, this approach fundamentally misaligns computational and statistical resources with our objective. Our focus is not on ranking all features but only the top-$k$, a crucial distinction that should guide algorithm design. Analogous to sparse modeling in high-dimensional statistics, we expect many features to be noise and are uninterested in learning their ranks \citep{10.5555/2834535}. However, estimating and sorting importance metrics expends valuable statistical resources (e.g., data samples) and computational effort on all features, including noise features. This misallocation is particularly problematic when distinguishing between features of similar importance levels, where more precise estimation is needed. 

To address these challenges, we develop a two-stage approach that combines efficient feature importance estimation with adaptive refinement. We first introduce RAMP (Ranked Attributions with MiniPatches), which leverages ensemble learning principles through random subsampling of both features and observations. We then extend this to RAMPART (RAMP And Recursive Trimming), which progressively focuses computational resources on the most promising features.

\subsection{Ranked Attributions with MiniPatches (RAMP)}

We begin by assuming access to a feature importance ranking procedure $\mathcal{M}$. This procedure $\mathcal{M}: \mathbb{R}^{n \times m} \times \mathbb{R}^n \mapsto \{0, \dots, m - 1\}$ takes as input a subset of the data (i.e., a ``minipatch'') $(\mathbf{X}_{I,F}, \mathbf{Y}_I)$ where $I \subseteq [N], |I| = n$ is a subsample of observations and $F \subseteq [M], |F| = m$ is a subsample of features, and returns rank estimates ${\tilde{r}_j}$ for features $j \in [F]$. In practice, this ranking procedure $\mathcal{M}$ typically involves: (i) fitting a predictive model to the minipatch, (ii) computing feature importance scores using a specific attribution method, and (iii) sorting these scores by magnitude to output ranks.

Given $\mathcal{M}$, we formalize RAMP in Algorithm \ref{algorithm: RAMP}. The procedure operates by generating numerous minipatches, with each consisting of a different random subsample of both observations and features. For each minipatch, RAMP applies $\mathcal{M}$ to obtain feature importance rank estimates $\tilde{r}_j^b$, which are then averaged across all minipatches where each feature appears. This minipatch ensembling approach reduces variance while breaking harmful correlation patterns between features \citep{gan_model-agnostic_2022}. The final step sorts the averaged ranks $\bar{r}_{j}$ using order statistics $\bar{r}_{(1)}, \dots, \bar{r}_{(M)}$ to obtain the final rankings.

\begin{algorithm}
\caption{Ranked Attributions with MiniPatches (RAMP)}
\label{algorithm: RAMP}
\begin{algorithmic}[1]
\REQUIRE Ranking procedure $\mathcal{M}$, dataset $\mathcal{D} = \{(\mbf{x}_i, y_i)\}_{i = 1}^{N} \equiv (\mathbf{X}, \mathbf{Y})$, number of minipatches $B$, data subsample size $n < N$, feature subsample size $m < M$
\ENSURE Estimated feature ranks ${\hat{r}_1, \dots, \hat{r}_M}$
\FOR{$b \in [B]$}
\STATE $I_b \gets$ randomly subsample $n$ observations $\subset [N]$
\STATE $F_b \gets$ randomly subsample $m$ features $\subset [M]$
\STATE $\tilde{r}_j^b \gets \mathcal{M}(\mathbf{X}_{I_b, F_b}, \mathbf{Y}_{I_b})_j$ for $j \in F_b$
\ENDFOR
\STATE For all $j \in [M]$, set $\; \bar{r}_j \gets \dfrac{\sum_{b \in [B]: j \in F_b} \tilde{r}_j^b}{\sum_{b \in [B]} \indic{j \in F_b}}$
\RETURN $\hat{r}_j = (i: \bar{r}_{(i)} = \bar{r}_j) - 1$
\end{algorithmic}
\end{algorithm}

Importantly, RAMP serves as a meta-algorithm that improves the accuracy of feature rankings regardless of the specific attribution method employed. More specifically, each ranking procedure $\mathcal{M}$ produces its own importance measure $\phi_j$ and hence ranks $r_j$, but ordinary estimates of these quantities typically suffer from instability in high-dimensions due to feature correlations and sampling variability \citep{chamma_statistically_2023, kelodjou_shaping_2024}. Given the choice of ranking procedure $\mathcal{M}$, RAMP improves the estimation of the corresponding ranks $r_j$ by ensembling across diverse minipatches, effectively reducing variance while maintaining the statistical properties of the underlying importance measures \citep{gan_model-agnostic_2022, yao_feature_2020}. This approach provides more stable and accurate approximations of the true feature importance ranking, while maintaining the flexibility to accommodate any feature importance procedure.

\subsection{RAMP And Recursive Trimming (RAMPART)}
\label{subsection: RAMPART}
Though RAMP provides an improved foundation for feature ranking, the uniform treatment of all features within RAMP is not ideal when our primary interest is in the top-$k$ ranked features. To address these limitations, we develop RAMPART (Ranked Attributions with MiniPatches And Recursive Trimming), an adaptive framework that builds upon the minipatch sampling from RAMP while incorporating ideas from the sequential halving literature. Our approach is inspired by the pioneering work of \citet{karnin_almost_2013} on successive halving algorithms, as well as recent advances in batched sequential halving \citep{jun_top_2016, koyamada_batch_2024} for the fixed batch setting. While these methods were originally developed for best-arm identification in multi-armed bandits, we adapt their core insight of progressive resource allocation to the feature ranking context.

RAMPART operates by iteratively applying RAMP to an increasingly focused set of features. In each round, it identifies and retains the more promising half of the features while eliminating those less likely to be in the top-\textit{k} set. This adaptive strategy, formalized in Algorithm \ref{algorithm: RAMPART}, recursively trims the feature pool size, enabling more accurate rank estimation in later rounds where fine-grained distinctions become crucial. The number of iterations is carefully chosen to ensure the final feature pool size aligns with our target \textit{k}. By concentrating computational resources on the most relevant features, RAMPART more efficiently spends its resources, distinguishing between similarly-ranked important features. This adaptive resource allocation is a critical advantage over traditional approaches that uniformly evaluate all features since features that survive to later rounds in RAMPART are evaluated more frequently, enabling increasingly precise rank estimates where they matter most. 

\begin{algorithm}
\caption{Ranked Attributions with MiniPatches And Recursive Trimming (RAMPART)}
\label{algorithm: RAMPART}
\begin{algorithmic}[1]
\REQUIRE Ranking procedure $\mathcal{M}$, dataset $\mathcal{D} = \{(\mathbf{x}_i, y_i)\}_{i = 1}^{N}$, number of minipatches $B$, data subsample size $n < N$, feature subsample size $m < M$,  top features $k$
\ENSURE Estimated feature ranks $\hat{r}_1, \ldots, \hat{r}_M$
\STATE $T \gets \lfloor \log_2 M \rfloor - \lceil \log_2 k \rceil + 1$
\STATE $\mathcal{C}_1 \gets [M]$
\FOR{$t \in \{1, \dots, T\}$}
\STATE $\hat{r}^t_1, \dots, \hat{r}^t_{|\mathcal{C}_t|} \gets$ RAMP$_{B, n, m}(\mathcal{C}_t$) 
\STATE $\mathcal{C}_{t + 1} \gets \{\hat{\tau}^t_1, \ldots, \hat{\tau}^t_{|\mathcal{C}_t|/2}\}$
\ENDFOR
\RETURN $\hat{r}_j = \hat{r}_j^T$ if $j \in \mathcal{C}_T$ otherwise $\hat{r}_j = k$
\end{algorithmic}
\end{algorithm}

\section{Theoretical Analysis}
\label{section: theoretical analysis}

In this section, we show theoretical guarantees for RAMP and RAMPART. We also prove that RAMPART achieves performance superior to that of RAMP under mild assumptions on the properties of the ranking procedure $\mathcal{M}$. 

\begin{assumption}\textbf{Unique Top-\textit{k} Ranks.}
\label{assumption: no ties}
For any two features $j, j' \in [M]$ where at least one feature has true rank smaller than $k$, either $r_j > r_{j'}$ or $r_{j'} > r_{j}$.
\end{assumption}

While this assumption might appear restrictive at first glance, it only requires distinct ranks among the top-\textit{k} features of interest. Ties are permitted among features outside this set. This reflects real-world scenarios where we need to distinguish among the most important features but can tolerate ambiguity in the ordering of less relevant or null features. 

\begin{assumption}
\label{assumption: consistency}
    \textbf{Rank Consistency.} For any two features $j, j' \in \mathcal{S} \subseteq [M], |\mathcal{S}| \ge m$ with $r_j < r_{j'}$ that are sampled in the same minipatch,
    \[\mathbf{P}\big(\tilde{r}_j < \tilde{r}_{j'}\big| j, j' \in F\big) \ge p > \dfrac{1}{2}\]
    where the probability is taken over all minipatches $F$ of size $m$ in $\mathcal{S}$ such that $j$ and $j'$ are sampled together.
\end{assumption}

This consistency assumption requires that our ranking procedure $\mathcal{M}$ performs better than random guessing when comparing features within the same minipatch on average. If a model cannot reliably order features when evaluated together, it cannot be expected to produce meaningful relative rankings. In particular, we only require probabilistic consistency, allowing for errors in individual comparisons.

\begin{assumption}
\label{assumption: unbiased}
    \textbf{Unbiased Ordering.} For any two features $j, j' \in \mathcal{S} \subseteq [M], |\mathcal{S}|\ge m$ with $r_j < r_j'$,
    \[\mathbf{E}\big[\tilde{r}_j \big| j \in F, j' \notin F\big]  < \mathbf{E}\big[\tilde{r}_{j'} \big| j' \in F, j \notin F\big]\]
    where expectations are taken over minipatches in $\mathcal{S}$ that sample one feature but not the other.
\end{assumption}

This assumption says that rank comparisons should remain informative across different minipatches. It requires that features of higher importance tend to receive better ranks relative to less important features, even when they appear in separate samples rather than being directly compared. This property is especially relevant for our minipatch approach, since we aggregate rank estimates across many different subsamples where not all pairs of features appear together.

\begin{assumption}
\label{assumption: bounded deviation}
\textbf{Bounded Deviation.} There exists a universal constant $C > 0$ such that for any two features $j, j' \in \mathcal{S} \subseteq [M], |\mathcal{S}|\ge m$ with $r_j < r_{j'}$ that are sampled in the same minipatch, 
\[\mathbf{E}[\tilde{r}_{j'} - \tilde{r}_j | \tilde{r}_{j'} > \tilde{r}_j] - \mathbf{E}[\tilde{r}_j - \tilde{r}_{j'} | \tilde{r}_j > \tilde{r}_{j'}] > C\]
where expectations are taken over all minipatches $F$ of size $m$ in $\mathcal{S}$ such that $j$ and $j'$ are sampled together.
\end{assumption}

This final assumption ensures that correctly ordered features are separated by a larger margin than incorrectly ordered ones. Specifically, when a more important feature is ranked above a less important one, their expected rank difference exceeds the expected difference when incorrectly ordered by at least some small positive constant $C$. This property provides stability in our estimates where ranking errors have less impact than correct orderings, allowing us to recover true feature ordering through ensembling. Together, these assumptions enable theoretical guarantees for our algorithms.

\begin{theorem}
\label{theorem: RAMP}
    Under Assumptions \ref{assumption: no ties}-\ref{assumption: bounded deviation},
    if the number of minipatches satisfies
    \[B_{\text{RAMP}} = \mathcal{O}\big(\dfrac{M^3}{m}\ln(\dfrac{kM}{\delta})\big),\]
    then with probability at least $1 - \delta$, RAMP will correctly rank all top-$k$ features: $\hat{r}_{\tau_j} = r_{\tau_j}$ for $j \in \{1, \dots, k\}$.
\end{theorem}

We defer the proof to Appendix \ref{appendix: proof of theorem RAMP}. Building on these results, we now show that RAMPART achieves stronger performance guarantees while requiring the same order of computational complexity.

\begin{theorem}
\label{theorem: RAMPART}
    Suppose $T \ge 3$. Then under assumptions \ref{assumption: no ties}-\ref{assumption: bounded deviation},
    there exist choices of $\{B_t\}_{t=1}^T$ such that RAMPART correctly identifies the top-$k$ features with probability at least $1 - \delta / 2 - \delta / T$ using the same order of total minipatches as RAMP: $\sum_{t = 1}^{T}B_t \sim B_{\text{RAMP}}$.
\end{theorem}

We defer the proof to Appendix \ref{appendix: proof of theorem RAMPART}. While both algorithms require the same order of total minipatches, RAMPART achieves the correct ranking with higher probability. Furthermore, our assumptions are substantially weaker compared to existing work. Traditional multi-armed bandit approaches require rewards to be independent \citep{russo_simple_2020}, while recent works on variable selection have imposed much stronger identifiability assumptions demanding each arm's optimal reward distribution be uniformly separated from all other possibilities \citep{liu_variable_2023}. In contrast, our assumptions only require probabilistic consistency and unbiased ordering, allowing for substantial noise and correlation between features.

\section{Empirical Studies}
\label{section: empirical studies}

In this section, we demonstrate the empirical performance of RAMP and RAMPART through a series of carefully designed experiments and a real-data case study. 

\subsection{Comparative Simulation Studies}
\label{subsection: comparative}

We design a comprehensive simulation framework to evaluate feature ranking methods across numerous settings. We first generate data from multivariate normal distributions $\mathbf{x}_i \in \mathbb{R}^M$ under two covariance structures: identity ($\Sigma = I$) or autoregressive ($\Sigma_{i,j} = \rho^{|i-j|}$ with $\rho = 0.5$). For both covariance structures, we assign non-zero coefficients $\beta_i = \gamma(10 - i + 1)$ to the first ten features ($i = 1, \ldots, 10$), with all others set to zero. The parameter $\gamma \in \{0.03, 0.05, 0.10, 0.20, 0.50\}$, referred to as signal-to-noise ratio (SNR) in our subsequent analyses, controls the separation between coefficient values, with smaller values creating more challenging ranking problems. 

We then construct four distinct  scenarios: linear regression, nonlinear additive regression, linear classification, and nonlinear additive classification. For regression scenarios, we generate responses as $y_i = f(\mathbf{x}_i) + \epsilon_i$ with $\epsilon_i \sim \mathcal{N}(0,1)$; for classification, $y_i \sim \text{Ber}(1/(1+e^{-f(\mathbf{x}_i)}))$. The function $f$ distinguishes linear scenarios, where $f(\mathbf{x}_i) = \mathbf{x}_i^T\beta$, from nonlinear additive scenarios, where $f(\mathbf{x}_i) = \sum_{j=1}^{M} \beta_j g_j(\mathbf{x}_{i,j})$, with $g_j(x) = \cos^{j+1}(x)$ for $j \in \{1,...,5\}$ and $g_j(x) = \sin^{j-4}(x)$ for $j \in \{6,...,10\}$. All features $X_j$ in the linear scenarios and $g_j(X_j)$ in the nonlinear scenarios are standardized to zero mean and unit variance, ensuring that coefficient magnitudes $|\beta_m|$ directly reflect the contribution of the features and function as the natural ground truth importance measures.

In each scenario, we employ task-appropriate predictions models --- namely, OLS and logistic regression for the linear regression and classification scenarios, respectively, and random forests (100 trees) for the nonlinear regression and classification tasks. Additionally, we employ neural networks across all scenarios, configured as regressors for regression tasks and classifiers (with final sigmoidal activation) for classification tasks. All neural networks have a consistent two-layer architecture with $M$ hidden units and ReLU activation trained to convergence.

As comparison methods, we first include a baseline (model-specific) feature importance (FI) method for each prediction model: absolute coefficients for OLS and logistic regression, Mean Decrease in Impurity (MDI) for random forests, and Integrated Gradients for neural networks \citep{sundararajan_axiomatic_2017}. We also evaluate two popular model-agnostic approaches. First, we apply SHAP with architecture-specific variants (LinearSHAP, TreeSHAP, or GradientSHAP) \citep{lundberg_unified_2017, lundberg_local_2020}, computing global scores by averaging local attribution scores across observations. Second, we assess permutation importance by measuring the average change in prediction error over 100 random permutations on a held-out set of size $N/2$. For all methods, we obtain feature rankings by sorting importance scores by magnitude.

To ensure a fair and direct comparison, we implement RAMP and RAMPART using the same baseline (model-specific) feature importance method, described previously, for each prediction model: absolute coefficients for OLS and logistic regression, MDI for decision trees, and Integrated Gradients for small neural networks (two-layer neural network with $5m$ hidden units and ReLU activation trained for 5 epochs on each minipatch). This implementation ensures direct comparability, as performance differences stem solely from our algorithmic framework rather than from variations in feature importance attributions. Both RAMP and RAMPART use minipatches with $n = 125$ observations and $m = 10$ features. For our experiments with $M = 500$ dimensions and $k = 10$ target features, RAMPART requires 6 halving iterations with 2000 minipatches per iteration, while RAMP uses 10000 total minipatches, maintaining comparable computational budgets. Note we omit \citet{goldwasser_statistical_2025} as their approach requires repeatedly resampling Shapley estimates for top features when statistical tests fail, making it prohibitively expensive in high dimensions.

We evaluate ranking performance using Rank-Biased Overlap (RBO) with $\rho = 0.7$ \citep{webber_similarity_2010}, which naturally prioritizes accuracy at higher ranks through geometrically decreasing weights, making it appropriate for our top-$k$ ranking task: \[\text{RBO}_{\rho}:= (1 - \rho) \cdot \sum_{s = 1}^k \rho^{s - 1} \dfrac{|\{\hat{\tau}_i\}_{i = 1}^s \cap \{\tau_i\}_{i = 1}^s|}{s}\]
For each experimental setting, we conduct 100 simulations with random seeds fixed across all methods, averaging the resulting RBO scores and reporting standard error bars in our plots. Figure \ref{fig: Combined_500} demonstrates that RAMPART consistently yields the most accurate feature rankings across both regression and classification tasks, highlighting the benefits of adaptive resource allocation. RAMPART's advantage becomes particularly pronounced at higher signal-to-noise ratios (SNR $\ge 0.1$) and under autoregressive covariance, showcasing robust performance even with correlated features. While performance naturally decreases in nonlinear additive settings for all methods, RAMPART maintains its relative advantage. These results demonstrate that RAMPART's adaptive allocation strategy provides substantial practical benefits across diverse modeling scenarios. Additional results for higher dimensions ($M = 1000$ and $M = 2000$) with fixed sample size ($n = 250$) are provided in Appendix \ref{appendix: additional simulation results}, further showcasing RAMPART's robustness in increasingly challenging high-dimensional settings.

To further demonstrate how accurate feature rankings translate to improved predictive performance, we also conducted ablation studies using the same simulation setup as in our identity covariance classification experiments from Figure \ref{fig: Combined_500}. Here, using the same predictive models and configurations as before, we selected two representative signal-to-noise ratios (SNR = $0.06$ and SNR = $0.5$), split the data into a 70/30 train-test split, and assessed the model's test prediction performance as the top-ranked features are progressively added in the model in order of their estimated importance rankings. In Figure \ref{fig: ablation}, we illustrate how classification error decreases as the top-ranked features are progressively added as predictors in the model. This comparison reveals the practical impact of feature importance ranking accuracy on prediction performance, with RAMPART's superior rankings consistently yielding the lowest classification errors from models trained using only the top-ranked features. Additional simulation results for ablation studies with higher dimensions can be found in Appendix \ref{appendix: additional simulation results}, which further verify RAMPART's robustness in increasingly challenging high-dimensional scenarios.

\begin{figure} 
    \centering
    \includegraphics[width=\textwidth]{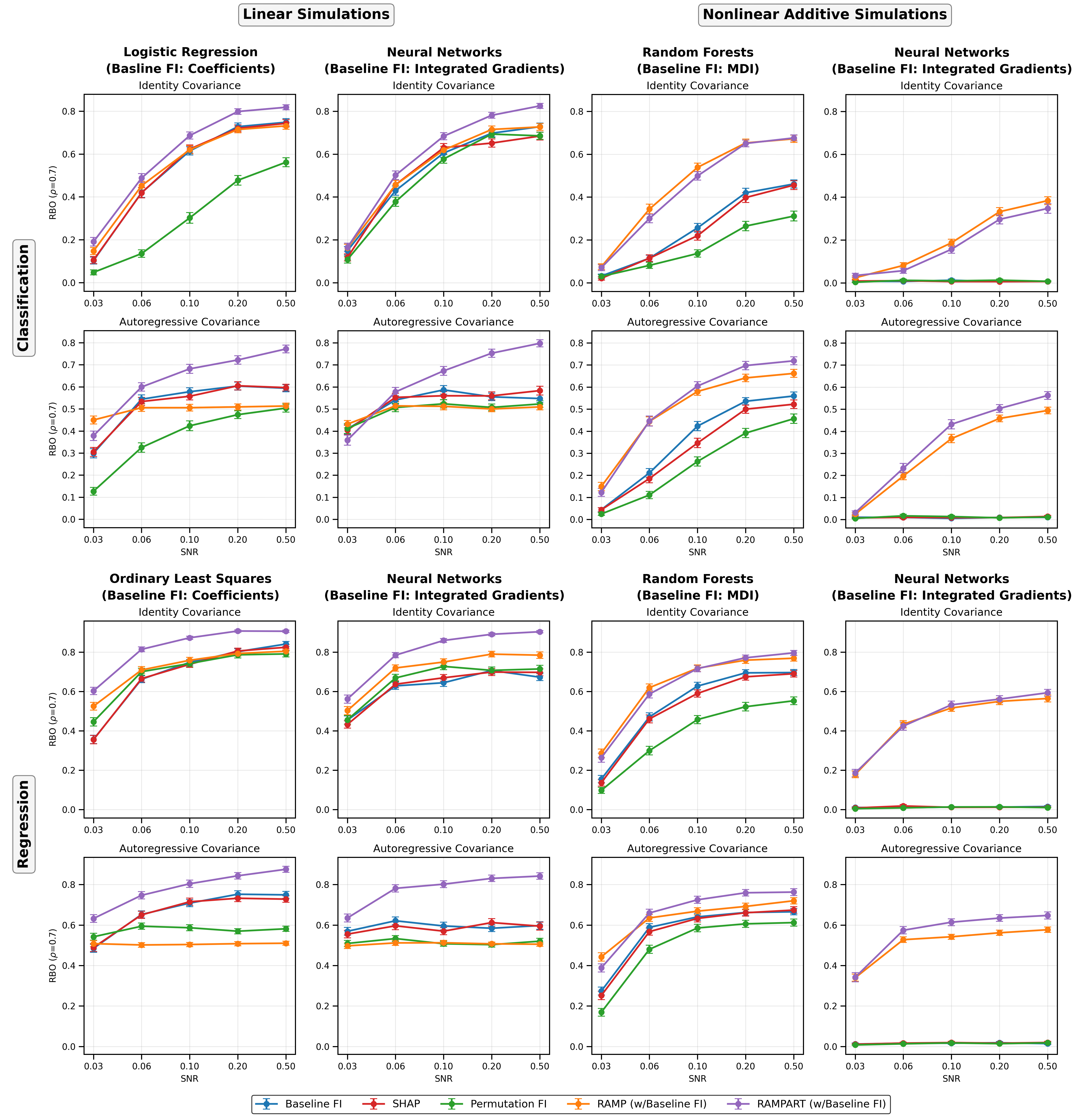}
    \caption{Feature importance ranking accuracy for classification (top) and regression (bottom)  ($M = 500$)} 
    \label{fig: Combined_500}
\end{figure}

\begin{figure}
    \centering
    \includegraphics[width=\textwidth]{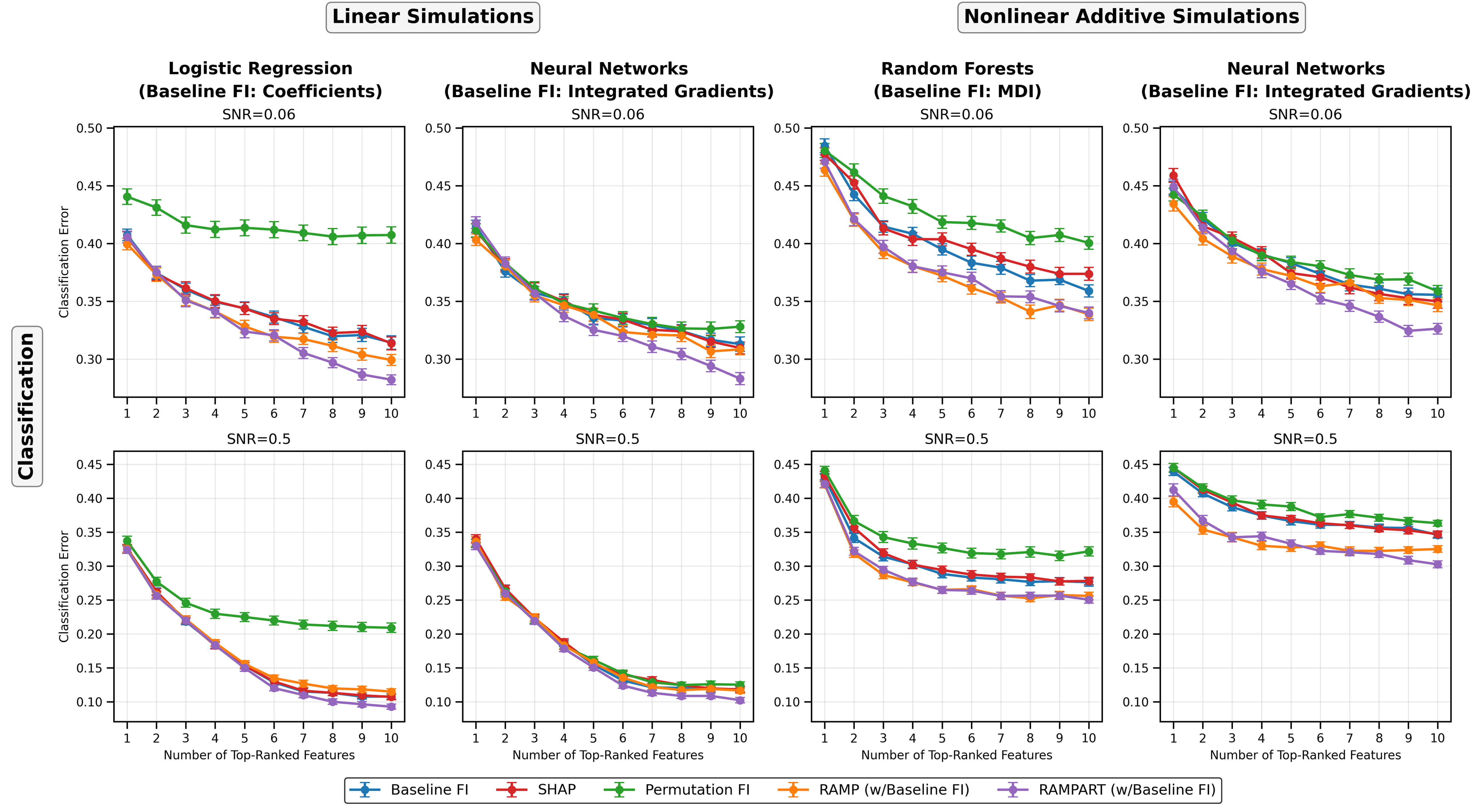}
    \caption{Classification error vs. number of top-ranked features used as predictors in ablation simulation with identity covariance $(M = 500)$ }
    \label{fig: ablation}
\end{figure}

\subsection{Cancer Genomics Case Study}

We lastly demonstrate RAMPART's effectiveness through a cancer genomics case study on drug response modeling. The goal is to develop predictive models of therapeutic cancer drug efficacy based on patients' genomic profiles. We focus specifically on ranking the most important genes driving drug response, providing insights into biological mechanisms of drug action and informing future therapeutic developments.

For our analysis, we predict response to PD-0325901 (an MEK inhibitor) across $N = 259$ human cancer cell lines using their RNASeq gene expression profiles ($M = 1104$ genes) from the Cancer Cell Line Encyclopedia (CCLE) \citep{barretina2012cancer} (see Appendix~\ref{appendix: CCLE} for details). As before, we apply multiple feature ranking methods to identify the top 10 response-driving genes. Our baseline approach uses a random forest regressor (200 trees) with Mean Decrease in Impurity (MDI). We also compute TreeSHAP values on the same random forest model and evaluate permutation importance by measuring prediction error changes over 100 random permutations on a 50/50 train-test split. For RAMP and RAMPART, we use MDI with regression trees as the minipatch ranking procedure $\mathcal{M}$, with minipatch parameters $m = 10$ and $n = 100$. RAMPART uses 4000 minipatches per iteration, while RAMP used 20000 total minipatches. Table~\ref{table: top 10 genes} presents the resulting gene rankings from all methods.

Since real-world genomic studies lack ground-truth feature importance rankings, we validate our top-10 gene findings using established biological knowledge. This biological validation approach follows established precedent in the feature importance literature \citep[e.g., see][]{lundberg_local_2020, NEURIPS2020_c7bf0b7c} where feature importance on genomic data are validated through connections to prior medical research. Along these lines, there are several key biological findings from this case study. First, all 10 genes identified by RAMPART have been previously implicated in biological pathways involving PD-0325901 and related cancers (Table~\ref{table: case study}). 
In particular, the top-ranked gene from RAMPART, \textit{TOR4A}, is a known oncogene for glioma \citep{wang2022delta}, a cancer for which PD-0325901 is currently being tested in clinical trials \citep{vinitsky2022lgg}. 
Second, the main pathway affected by PD-0325901 
is the MEK/ERK signaling pathway, which regulates cell proliferation, differentiation, and survival. Several of the genes identified by RAMPART, including \textit{ETV4}, \textit{SPRY2}, and \textit{WNT5A}, are key factors in the MEK/ERK pathway \citep{oh2012etv1, milillo2015spry2, hasan2021wnt5a}.

Finally, we perform a gene ontology (GO) enrichment analysis \citep{ashburner2000gene, gene2023gene}, a standard bioinformatics approach for identifying the biological processes that are enriched (or over-represented) in a set of genes, to further validate the biological relevance of the top-10 genes identified by RAMPART. According to the GO enrichment analysis, the biological processes, ``regulation of transmembrane receptor protein serine/threonine kinase signaling pathway'' and ``regulation of cellular response to growth factor stimulus,'' are significantly enriched (FDR $p < 0.05$) in the top-10 genes from RAMPART. Notably, the serine/threonine kinase signaling pathway plays an essential role in the activation of MEK \citep{zheng1994activation}. In contrast, no GO biological processes are significantly enriched in the top-10 genes from any of the other competing methods. Additionally, of the four genes (i.e., \textit{SPRY2}, \textit{FERMT1}, \textit{WNT5A}, and \textit{NRROS}) linked to the identified GO biological processes, two are uniquely identified by RAMPART and have not been ranked in the top 10 by any other method. This GO enrichment analysis not only supports the biological relevance of RAMPART's identified genes, but also highlights the advantage of RAMPART over existing methods for feature importance ranking and high-impact scientific discovery.

\begin{table}[t]
\begin{center}
\begin{tabular}[t]{p{1cm}p{1.5cm}p{12.5cm}}
\toprule
\textbf{Rank} & \textbf{Gene} & \textbf{Connection to PD-0325901} \\
\midrule
1 & \textit{TOR4A} & Torsin family gene and oncogene for glioma and other cancers \citep{wang2022delta} \\
\cellcolor{gray!10}{2} & \cellcolor{gray!10}{\textit{ETV4}} & \cellcolor{gray!10}{Well-known ETS transcription factor regulated by MEK/ERK pathway \citep{oh2012etv1}}\\
3 & \textcolor{cornflowerblue}{\textit{SPRY2}} & Key inhibitor of the MEK/ERK pathway \citep{zheng1994activation, milillo2015spry2}\\
\cellcolor{gray!10}{4} & \cellcolor{gray!10}{\textit{GJB1}} & \cellcolor{gray!10}{Gap junction gene associated with increased tumor progression for various cancers \citep{aasen2016gap}} \\
5 & \textit{PYCARD} & Encodes key adaptor protein in inflammatory and apoptotic signaling pathways, playing dual roles in multiple cancers \citep{protti2020dual}\\
\cellcolor{gray!10}{6} & \cellcolor{gray!10}{\textcolor{cornflowerblue}{\textit{WNT5A}}} & \cellcolor{gray!10}{Wnt signaling pathway gene which can enhance MEK/ERK pathway \citep{zheng1994activation, hasan2021wnt5a}}\\
7 & \textcolor{cornflowerblue}{\textit{FERMT1}} & Regulates Ser/Thr kinase signaling pathway, activating MEK \citep{zheng1994activation}\\
\cellcolor{gray!10}{8} & \cellcolor{gray!10}{\textcolor{cornflowerblue}{\textit{NRROS}}} & \cellcolor{gray!10}{Regulates Ser/Thr kinase signaling pathway, activating MEK \citep{zheng1994activation}}\\
9 & \textit{LYZ} & Encodes lysozyme; shown to exhibit aberrant expression in tumor cells \citep{gu2023aberrant}\\
\cellcolor{gray!10}{10} & \cellcolor{gray!10}{\textit{NPAS2}} & \cellcolor{gray!10}{Regulator of circadian rhythms and tumor suppressor involved in DNA damage response \citep{hoffman2008circadian}}\\
\bottomrule
\end{tabular}
\caption{Top-10 genes from RAMPART and their connections to PD-0325901. Genes highlighted in blue were identified by the GO enrichment analysis to be involved in the ``regulation of transmembrane receptor protein serine/threonine (Ser/Thr) kinase signaling pathway'' and ``regulation of cellular response to growth factor stimulus.''}
\label{table: case study}
\end{center}
\end{table}

\section{Discussion}
\label{section: discussion}

In this paper, we introduced RAMPART, a novel framework that achieves accurate feature ranking in high-dimensional settings by combining minipatch ensembles with recursive trimming. Our genomics case study demonstrates how precise feature ranking can drive scientific discovery by identifying key drivers in complex biological processes and guiding future research. RAMPART opens several promising research directions. First, handling high feature correlation remains challenging and requires methods that account for complex dependency structures. Second, adapting minipatch sizes to scale with the shrinking feature pool could enable more precise comparisons while maintaining computational efficiency. Third, while we provide initial theoretical support for top-$k$ feature importance ranking, refining assumptions and bounds could better explain the empirical performance gap between RAMPART and RAMP. Finally, our minipatch ensemble framework offers a flexible foundation for exploring alternative adaptive sampling methods to further enhance top-$k$ ranking performance.

\bibliography{references}
\bibliographystyle{unsrtnat}

\appendix

\section{Proofs}
\label{appendix: proofs}

\subsection{Auxiliary Lemmas}

\begin{lemma}
\label{lemma: separability}
    Let $\mu_j = \mathbf{E}[\tilde{r}_j]$ be the expected rank of feature $j$ under the ranking procedure $\mathcal{M}$, where the expectation is taken over minipatches that sample feature $j$. Next, define 
    \[\Delta := \min_{j, j' \in [M], r_j < r_{j'}} \mu_{j'} - \mu_j\]
    to be the smallest difference between the expected rank estimates across all features in $[M]$. Then under assumptions \ref{assumption: consistency}, \ref{assumption: unbiased} and \ref{assumption: bounded deviation}, we have
    \[\Delta > (2p - 1) \cdot \big(\dfrac{m - 1}{M - 1}\big)\]
\end{lemma}

\begin{proof}
    First, observe that given the above assumptions, we necessarily have $\mu_j < \mu_{j'}$ whenever $r_j < r_{j'}$ as we have
    \begin{align*}
        \mu_j 
        &= \mathbf{E}[\tilde{r}_j | j \in F] \\
        &= \mathbf{E}[\tilde{r}_j | j \in F, {j'} \notin F] \cdot \mathbf{P}(j' \notin F \; | \; j \in F) + \mathbf{E}[\tilde{r}_j | j, j' \in F] \cdot \mathbf{P}(j' \in F \; | \; j \in F)\\
        &< \mathbf{E}[\tilde{r}_{j'} | j' \in F, j \notin F] \cdot \mathbf{P}(j \notin F \; | \; j' \in F) + \mathbf{E}[\tilde{r}_{j'} | j, j' \in F] \cdot \mathbf{P}(j \in F \; | \; j' \in F)\\
        &= \mathbf{E}[\tilde{r}_{j'} | j' \in F]\\
        &= \mu_{j'}    
    \end{align*}
    where the inequality follows from knowing that $\mathbf{E}[\tilde{r}_{j'} | j' \in F, j \notin F] > \mathbf{E}[\tilde{r}_{j} | j \in F, j' \notin F]$ by assumption \ref{assumption: unbiased}, that $\mathbf{P}(j' \notin F \; | \; j \in F) = \mathbf{P}(j \notin F \; | \: j' \in F)$, and that $\mathbf{E}[\tilde{r}_{j'} | j, j' \in F] > \mathbf{E}[\tilde{r}_j | j, j' \in F]$ by assumptions \ref{assumption: consistency} and \ref{assumption: bounded deviation} since
    \begin{align*}
         &\mathbf{P}(\tilde{r}_{j'} > \tilde{r}_j|j, j' \in F) \cdot \big(\mathbf{E}[\tilde{r}_{j'} | \tilde{r}_{j'} > \tilde{r}_j]- \mathbf{E}[\tilde{r}_j | \tilde{r}_{j'} > \tilde{r}_j]\big) \\> &\mathbf{P}(\tilde{r}_j > \tilde{r}_{j'}|j, j' \in F) \cdot \big(C + \mathbf{E}[\tilde{r}_j | \tilde{r}_j > \tilde{r}_{j'}] - \mathbf{E}[\tilde{r}_{j'} | \tilde{r}_j > \tilde{r}_{j'}]\big)
    \end{align*}
    and rearranging gives the desired result. It follows that for any $j, j' \in [M]$ with $r_j < r_{j'}$
    \begin{align*}
        \mu_{j'} - \mu_j 
        &= \mathbf{E}[\tilde{r}_{j'} | j, j' \in F] \cdot \mathbf{P}(j \in F \; | \; j' \in F) + \mathbf{E}[\tilde{r}_{j'} | j' \in F, j \notin F] \cdot \mathbf{P}(j \notin F \; | \; j' \in F) \\
        &- \mathbf{E}[\tilde{r}_j | j, j' \in F] \cdot \mathbf{P}(j' \in F \; | \; j \in F) - \mathbf{E}[\tilde{r}_j | j \in F, j' \notin F) \cdot \mathbf{P}(j' \notin F \; | \; j \in F)\\
        &> \mathbf{E}[\tilde{r}_{j'} - \tilde{r}_j | j, j' \in F] \cdot \mathbf{P}(j \in F \; | \; j' \in F)\\
        &= \big(\mathbf{E}[\tilde{r}_{j'} - \tilde{r}_j | \tilde{r}_{j'} > \tilde{r}_j] \cdot \mathbf{P}(\tilde{r}_{j'} > \tilde{r}_j)+ \mathbf{E}[\tilde{r}_{j'} - \tilde{r}_j | \tilde{r}_j > \tilde{r}_{j'}] \cdot \mathbf{P}(\tilde{r}_j > \tilde{r}_{j'})\big) \cdot \mathbf{P}(j \in F \; | \; j' \in F)\\
        &> \big(\mathbf{E}[\tilde{r}_{j'} - \tilde{r}_j | \tilde{r}_{j'} > \tilde{r}_j] \cdot p + (C - \mathbf{E}[\tilde{r}_{j'} - \tilde{r}_j | \tilde{r}_{j'} > \tilde{r}_j]) \cdot (1 - p)\big) \cdot \mathbf{P}(j \in F \; | \; j' \in F)\\
        &= (\mathbf{E}[\tilde{r}_{j'} - \tilde{r}_j | \tilde{r}_{j'} > \tilde{r}_j] \cdot (2p - 1) + C(1 - p)) \cdot \mathbf{P}(j \in F \; | \; j' \in F)\\
        &> (2p - 1) \cdot \mathbf{P}(j \in F \; | \; j' \in F)\\
        &> (2p - 1) \cdot \dfrac{(m - 1)}{(M - 1)}
    \end{align*}
    where the first equality holds by the law of total conditional expectation, the first inequality holds by  assumption \ref{assumption: unbiased}, the second inequality holds by assumption \ref{assumption: bounded deviation}, and the second-to-last inequality holds since the difference between estimates ranks is at least $1$ and that $C > 0$.
\end{proof}

\begin{lemma} 
\label{lemma: gap} 
For iteration $t$ of algorithm \ref{algorithm: RAMPART}, define the minimum separation gap 
\[\Delta_t := \min_{j, j' \in \mathcal{C}_t, r_j < r_{j'}}\mu_j^t - \mu_{j'}^t\] 
where $\mu_j^t$ is the expected rank estimate of feature $j$ with respect to the feature pool $\mathcal{C}_t$. Under Assumptions \ref{assumption: consistency}-\ref{assumption: bounded deviation}, we have 
\[\Delta_{t} > 2^{t - 1} \cdot (2p - 1) \cdot (\dfrac{m - 1}{M - 1})\] 
for any iteration $t$. 
\end{lemma}

\begin{proof}
    At any iteration $t$, we know from lemma \ref{lemma: separability} that 
\[\Delta_{t + 1} > (2p - 1)\cdot \dfrac{m - 1}{|C_{t + 1}| - 1} = (2p - 1) \cdot \dfrac{m - 1}{M/2^t - 1} \ge 2^t \cdot (2p - 1) \cdot \dfrac{m - 1}{M - 1}\]
\end{proof}

\subsection{Proof of Theorem \ref{theorem: RAMP}}
\label{appendix: proof of theorem RAMP}

\begin{proof}
    Let $\mathcal{T} = \{\tau_1, \dots, \tau_k\}$ denote the set of top-$k$ features. In addition, for any feature $j \in [M]$ let $S_{j}$ denote the set of minipatches where feature $j$ is included. Next, for any given feature $j$ we order the minipatches in $S_j$ arbitrarily as $\kappa_1, \dots, \kappa_{n_j}$ where $n_j = |S_j|$ and define the martingale difference sequence indexed by
    \[X_s = \tilde{r}_j^{\kappa_s} - \mathbf{E}[\tilde{r}_j^{\kappa_s} | \mathcal{F}_{s - 1}]\]
    where $\mathcal{F}_{s - 1} = \sigma(\kappa_1, \dots, \kappa_{s - 1})$. We see that $X_s$ is $\mathcal{F}_s$-measurable and 
    \[\mathbf{E}[X_s | \mathcal{F}_{s - 1}] = \mathbf{E}[\tilde{r}_j^{\kappa_s} - \mathbf{E}[\tilde{r}_j^{\kappa_s} | \mathcal{F}_{s - 1}] | \mathcal{F}_{s - 1}] = 0\]
    We also trivially see that $|X_s| \leq m - 1$. Then, by Azuma-Hoeffding's inequality we can write
    \begin{align*}
        \mathbf{P}(\bar{r}_j - \mu_j \ge \epsilon)
        &= \mathbf{P}(M_{n_j} \ge n_j\epsilon)
        \leq \exp(-\dfrac{n_j^2\epsilon^2}{2n_j(m - 1)^2})
        = \exp(-\dfrac{n_j\epsilon^2}{2(m - 1)^2})
    \end{align*}
    where the first equality holds by observing that 
    \[M_{n_j} := \sum_{s = 1}^{n_j}X_s = \sum_{s = 1}^{n_j}\tilde{r}_j^{\kappa_s} - \mathbf{E}[\tilde{r}_j^{\kappa_s} | \mathcal{F}_{s - 1}] = n_j(\bar{r}_j - \mu_j)\]
    Next, since $n_j \sim \text{Binomial}(B, q)$ where $q = m / M$, we have by the Chernoff bound that
    \[\mathbf{P}(n_j \leq \dfrac{Bq}{2}) \leq \exp(-\dfrac{Bq}{8})\]
    Then, by the law of total probability we can write
    \begin{align*}
        \mathbf{P}(\bar{r}_j - \mu_j \ge \epsilon)
        &= \mathbf{P}(\bar{r}_j - \mu_j \ge \epsilon, n_j \leq \frac{Bq}{2}) + \mathbf{P}(\bar{r}_j - \mu_j \ge \epsilon, n_j \ge \frac{Bq}{2})\\
        &\leq \exp(- \dfrac{Bq}{8}) + \mathbf{P}(\bar{r}_i - \mu_i \ge \epsilon, n_i \ge \dfrac{Bq}{2})\\
        &\leq \exp(-\dfrac{Bq}{8}) + \exp(-\dfrac{Bq\epsilon^2}{4(m - 1)^2})
    \end{align*}
    As before, we define \[\Delta := \min_{j, j' \in [M], r_j < r_{j'}} \mu_{j'} - \mu_j\] and set $\epsilon = \Delta / 2$. By assumption \ref{assumption: no ties} this quantity is non-zero and well-defined. Then we can write for some top-$k$ feature $j$ and other feature $j'$ where $r_j < r_{j'}$ that
    \begin{align*}
        \mathbf{P}(\bar{r}_j \ge \bar{r}_{j'}) 
        &\leq \mathbf{P}(\bar{r}_j \ge \mu_j + \epsilon) + \mathbf{P}(\bar{r}_{j'} \leq \mu_{j'} - \epsilon)
    \end{align*}
    We see that the probability of any true top-$k$ feature not being ranked correctly is upper bounded by
    \[\mathbf{P}_{\text{err}} \leq \mathbf{P}\big(\bigcup_{j \in \mathcal{T}} \bigcup_{j': r_{j'} > r_j}\{\bar{r}_j \ge \bar{r}_{j'}\}\big)\]
    By the union bound, we see we want to find $B$ such that
    \begin{align*}
        \mathbf{P}_{\text{err}} \leq \sum_{j \in \mathcal{T}} \sum_{j': r_{j'} > r_j}\mathbf{P}(\bar{r}_j \ge \bar{r}_{j'}) \leq 2kM \cdot \big(\exp(-\dfrac{Bq}{8}) + \exp(-\dfrac{Bq\epsilon^2}{4(m - 1)^2})\big) \leq \delta
    \end{align*}
    One way to choose $B$ is such that each of the two exponentials $\leq \delta / 4kM$. Solving for these two conditions separately gives
    \begin{align*}
        B &\ge \max\big\{\dfrac{8}{q}\ln(\dfrac{4kM}{\delta}), \dfrac{4(m - 1)^2}{q\epsilon^2}\ln(\dfrac{4kM}{\delta})\big\} \\
        &= \dfrac{4(m - 1)^2}{q\epsilon^2}\ln(\dfrac{4kM}{\delta})\\
        &= \dfrac{16M(M-1)^2}{(2p - 1)^2m}\ln(\dfrac{4kM}{\delta})
    \end{align*}
    where $\epsilon = \Delta / 2$ and $\Delta \ge (2p - 1)(\dfrac{m - 1}{M - 1})$ by lemma \ref{lemma: separability}. Choosing $C = \dfrac{16}{(2p - 1)^2}$ yields the result.
\end{proof}

\subsection{Proof of Theorem \ref{theorem: RAMPART}}
\label{appendix: proof of theorem RAMPART}
\begin{proof}
    Let $N_t = M / 2^{t - 1}$ denote the feature pool size at iteration $t$, $q_t = m / N_t$, $\epsilon_t = \delta_t / 2$, and $\mathcal{D}_t$ be the features ranked in the bottom half at iteration $t$ to be discarded. We trivially see that $r_{j'} > r_j$ for any $j' \in \mathcal{D}_t$ and $j \in \mathcal{T}$. Furthermore, any top-$k$ feature $j$ at iteration $t \leq T - 1$ will survive onto the next round by being ranked in the upper-half of the feature pool. Then defining $E_{j, t}$ as the event that feature $j$ is incorrectly eliminated at iteration $t$, we can write for any $t \leq T - 1$ that
    \begin{align*}
        \mathbf{P}(E_{j, t}) 
        &\leq \mathbf{P}(\bigcup_{j' \in \mathcal{D}_t}\{\bar{r}_j^t \ge \bar{r}_{j'}^t\}) \\
        &\leq \sum_{j' \in \mathcal{D}_t} \mathbf{P}(\bar{r}_j^t \ge \bar{r}_{j'}^t)\\
        &\leq \dfrac{N_t}{2} \big(\mathbf{P}(\bar{r}_j^t \ge \mu_j^t + \epsilon_t) + \mathbf{P}(\bar{r}_{j'}^t \leq \mu_{j'}^t - \epsilon_t) \big)\\
        &= N_t \big(\exp(\dfrac{-B_tq_t}{8}) + \exp(-\dfrac{B_tq_t\epsilon_t^2}{4(m - 1)^2})\big)
    \end{align*}
    where the second inequality holds by the subadditivity of measure, and the last two (in)equalities hold by a similar argument found in appendix \ref{appendix: proof of theorem RAMP}. For each iteration $t \leq T - 1$, we see that the probability of error is upper bounded by
    \begin{align*}
        \mathbf{P}_{\text{err}}^t \leq \sum_{j \in \mathcal{T}} \mathbf{P}(E_{j, t})
    \end{align*}
    Whereas for the last iteration $T$, we need to correctly rank all top-$k$ features. Similarly following appendix \ref{appendix: proof of theorem RAMP}, we have
    \begin{align*}
        \mathbf{P}_{\text{err}}^T \leq \sum_{j \in \mathcal{T}}\mathbf{P}(E_{j, T}) \leq \sum_{j \in \mathcal{T}} \sum_{j': r_{j'} > r_j}\mathbf{P}(\bar{r}_j \ge \bar{r}_{j'})
    \end{align*}
    Next, we want to control each term $\mathbf{P}_{\text{err}}^t \leq \delta / 2T$ and $\mathbf{P}_{\text{err}}^T \leq \delta / T$ such that the total probability of error is bounded within $\mathbf{P}_{\text{err}} \leq \delta / 2 + \delta / T$. To achieve this, and again following a similar argument as outlined in appendix \ref{appendix: proof of theorem RAMP}, we can choose \[B_t \ge \dfrac{4(m-1)^2}{q_t\epsilon_t^2} \ln (\dfrac{4kTN_t}{\delta})\]
    for $t \leq T - 1$ and 
    \[B_T \ge \dfrac{4(m-1)^2}{q_t\epsilon_t^2} \ln(\dfrac{4kTN_T}{\delta}).\]
    Set $\Delta_1 = (2p - 1)\frac{m - 1}{M - 1}$
    Then following lemma \ref{lemma: gap}, we can write $B_{\text{RAMPART}}$ as
    \begin{align*}
        B_{\text{RAMPART}} 
        &\leq \sum_{t = 1}^{T} B_t = \dfrac{16(m - 1)^2}{m}\sum_{t = 1}^{T}\dfrac{N_t}{\Delta_t^2}\ln \big(\dfrac{4kTN_t}{\delta}\big)\\
        &= \dfrac{16(m-1)^2}{m\Delta_1^2} \sum_{t = 1}^{T} \dfrac{N_t}{(2^{t - 1})^2} \ln \big(\dfrac{4kTN_t}{\delta}\big) \\
        &= \dfrac{16(m - 1)^2}{m\Delta_1^2}\sum_{t = 1}^{T}\dfrac{M}{2^{2(t - 1) + t - 1}}\big(\ln(\dfrac{4kTM}{\delta}) - (t - 1)\ln 2\big) \\
        &= \dfrac{16(m - 1)^2}{m\Delta_1^2} \cdot \big(M\ln(\dfrac{4kTM}{\delta})\sum_{t = 1}^{T}\dfrac{1}{2^{3t - 3}} - M \ln 2 \sum_{t = 1}^{T}\dfrac{(t - 1)}{2^{3t - 3}}\big)\\
        &\leq \dfrac{16M(m - 1)^2}{m\Delta_1^2} \cdot \big(\dfrac{6}{5}\ln\big(\dfrac{4kTM}{\delta}\big) - \dfrac{1}{10}\ln 2\big) \\
        &\leq \dfrac{16M(m - 1)^2}{m\Delta_1^2} \cdot \big(\dfrac{6}{5}\ln\big(\dfrac{4k(\log_2(M) - \log_2(k) + 1)M}{\delta}\big)\big)\\
        &\sim \dfrac{16M(m - 1)^2}{m\Delta_1^2}  \cdot \big(\dfrac{6}{5}\ln(\dfrac{4kM}{\delta})\big)
    \end{align*}
    where in the second-to-last inequality follows from knowing that $T \ge 3$. 
\end{proof}

\section{Additional Simulation Results}
\label{appendix: additional simulation results}
In this section, we further validate RAMPART’s effectiveness and our theoretical analysis by examining how the probability of exactly recovering the true top-$k$ features scales with the total number of minipatches. We then assess robustness in higher dimensions ($M=1000,2000$) and analyze the impact of minipatch size on ranking accuracy.

\subsection{Theory Validation}
We first verify our theoretical results stated in Theorem \ref{theorem: RAMPART} by examining how the number of minipatches affects ranking performance of RAMP and RAMPART. We generate synthetic data with $N = 1000$ observations and $M = 160$ features drawn independently from a standard normal distribution $\mathbf{X} \sim \mathcal{N}(0, I)$ with $\mathbf{Y} = 0.22 \cdot \mathbf{X}\beta + \epsilon$, where $\epsilon$ is unit Gaussian noise. The coefficient vector $\beta$ is constructed to have four non-zero features with coefficients 4, 3, 2, 1. Since features are on the same scale, the magnitude of these coefficients directly determines the feature importance ordering, providing a clear ground truth for evaluating ranking performance. We compare RAMP and RAMPART with minipatch parameters $n = 80$ and $m = 20$ (the trimming process also terminates when there are 20 features remaining). To ensure a fair comparison, we allocate a total budget of $B$ minipatches across the five halving iterations for RAMPART, while using $5B$ minipatches for RAMP to match the total computation. We evaluate the empirical success probability averaged over 500 random trials, where success is defined as exactly recovering the ranks of the top four features.

\begin{figure}[hbt!]
\begin{center}
\includegraphics[width= 0.85\columnwidth]{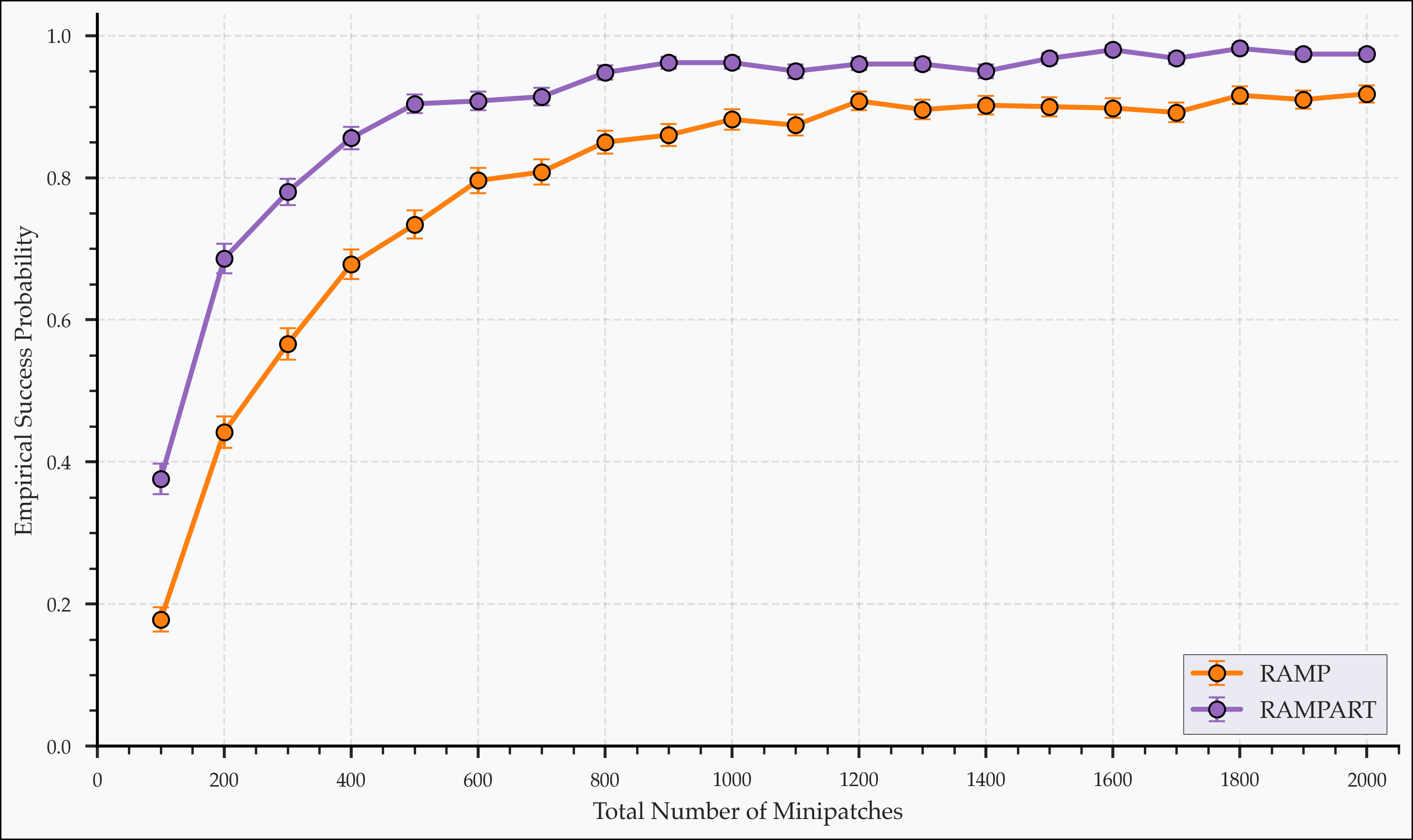}
\caption{Number of Minipatches vs Ranking Success Probability} \label{fig: num minipatch}
\end{center}
\end{figure}

Figure \ref{fig: num minipatch} shows the empirical success probability as a function of the total number of minipatches. In general, both methods demonstrate increasing accuracy with more minipatches, with RAMPART consistently achieves superior ranking accuracy compared to RAMP, validating the consistency guarantees of Theorems \ref{theorem: RAMP} and \ref{theorem: RAMPART}. 

While our theory shows RAMPART and RAMP require the same order of total minipatches, empirically RAMPART achieves better performance, particularly in the regime of 100-1000 minipatches where its adaptive resource allocation proves beneficial. This suggests that RAMPART's strategy of progressively focusing computation on promising features provides practical advantages beyond what our theoretical analysis captures. Bridging this gap to obtain sharper theoretical guarantees that better align with empirical performance remains an exciting direction for future work.

\subsection{Higher Dimensions}

For $M = 1000$ and $2000$, we extend the setup described in Section \ref{section: empirical studies} by adding more nonzero features while preserving the same coefficient structure. We make two implementation adjustments to accommodate higher dimensionality: (1) For baseline FI, SHAP, and feature permutation with random forests, we scale the number of trees (200 trees for $M = 1000$ and 400 trees for $M = 2000$) to ensure adequate feature space coverage; (2) For RAMPART and RAMP, we set $B = 4000$ and $B = 20000$ minipatches respectively, while keeping all other parameters unchanged.

The results for $M = 1000$ (Figures \ref{fig: claig1000} \& \ref{fig: regig1000}) demonstrate RAMPART's robust performance across both covariance structures, with its advantage becoming particularly pronounced under autoregressive covariance. RAMPART's sequential halving strategy proves especially effective in these higher-dimensional settings, showing remarkable stability compared to $M = 500$ while other methods exhibit noticeable degradation. At extreme dimensionality ($M = 2000$, Figures \ref{fig: claig2000} \& \ref{fig: regig2000}), RAMPART continues to excel, maintaining strong performance even at high signal-to-noise ratios under autoregressive covariance, while competing methods show significant accuracy drops. These results highlight RAMPART's unique ability to handle both high dimensionality and complex feature interactions through its adaptive resource allocation strategy.

We also extend our ablation studies to higher dimensions using the same methodology as Section \ref{subsection: comparative}, focusing on the identity covariance setting and examining columns 2 and 5 (SNR = $0.06$ and SNR = $0.5$) across both classification and regression tasks. As before, we maintain consistent experimental conditions by using the same predictive models, data generation processes, and hyperparameter configurations across all settings. For regression tasks, we measure the mean squared error (MSE) rather than classification error. Figures \ref{fig: ablation Reg_500}-\ref{fig: ablation Reg_2000} demonstrate that RAMPART maintains its performance advantage across all dimensions, affirming that RAMPART's adaptive resource allocation strategy remains valuable even in higher-dimensional settings where feature ranking becomes more challenging.

\subsection{Impact of Minipatch Size}

We also investigate the effect of varying feature subsample size $m$ within the RAMP framework, returning to the $M = 500$ setting with experimental conditions from Section \ref{section: empirical studies}. Figures \ref{fig: MP_Cla} and \ref{fig: MP_Reg} demonstrate that while performance remains stable across minipatch sizes for linear models under identity covariance, smaller minipatches ($m = 10$) consistently outperform larger ones ($m = 50$) in nonlinear additive and correlated settings. These results suggest that smaller minipatch sizes may better capture local feature interactions while avoiding noise from irrelevant features. While we fix $m = 10$ for RAMPART in this work, exploring adaptive minipatch sizes that scale with feature pool size across halving iterations remains an interesting direction for future research.

\begin{figure}[hbt!]
\begin{center}
\includegraphics[width= \textwidth]{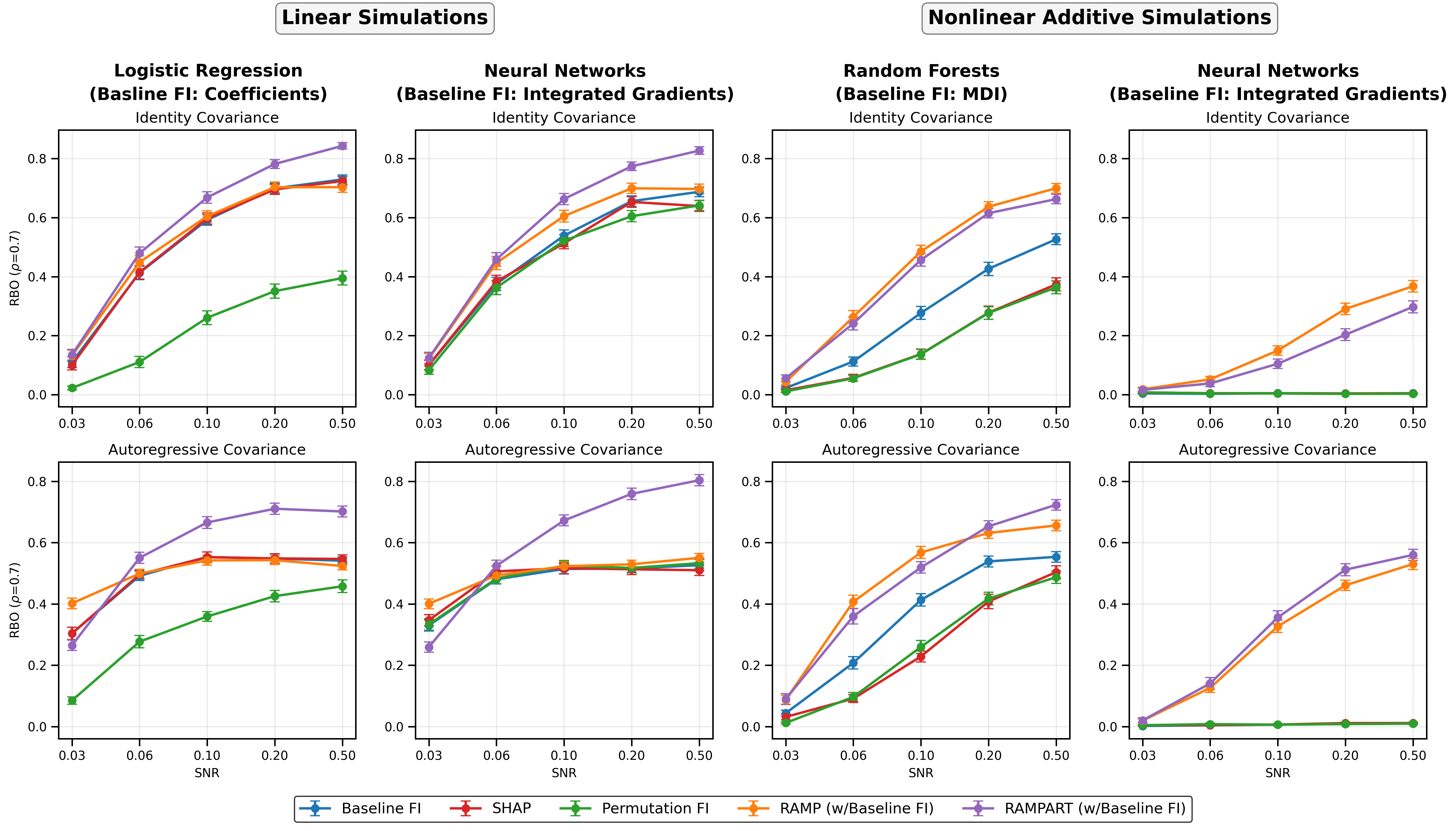}
\caption{Ranking accuracy (RBO with $\rho = 0.7$) for classification tasks  ($M = 1000$)} \label{fig: claig1000}
\end{center}
\end{figure}

\begin{figure}[hbt!]
\begin{center}
\includegraphics[width= \textwidth]{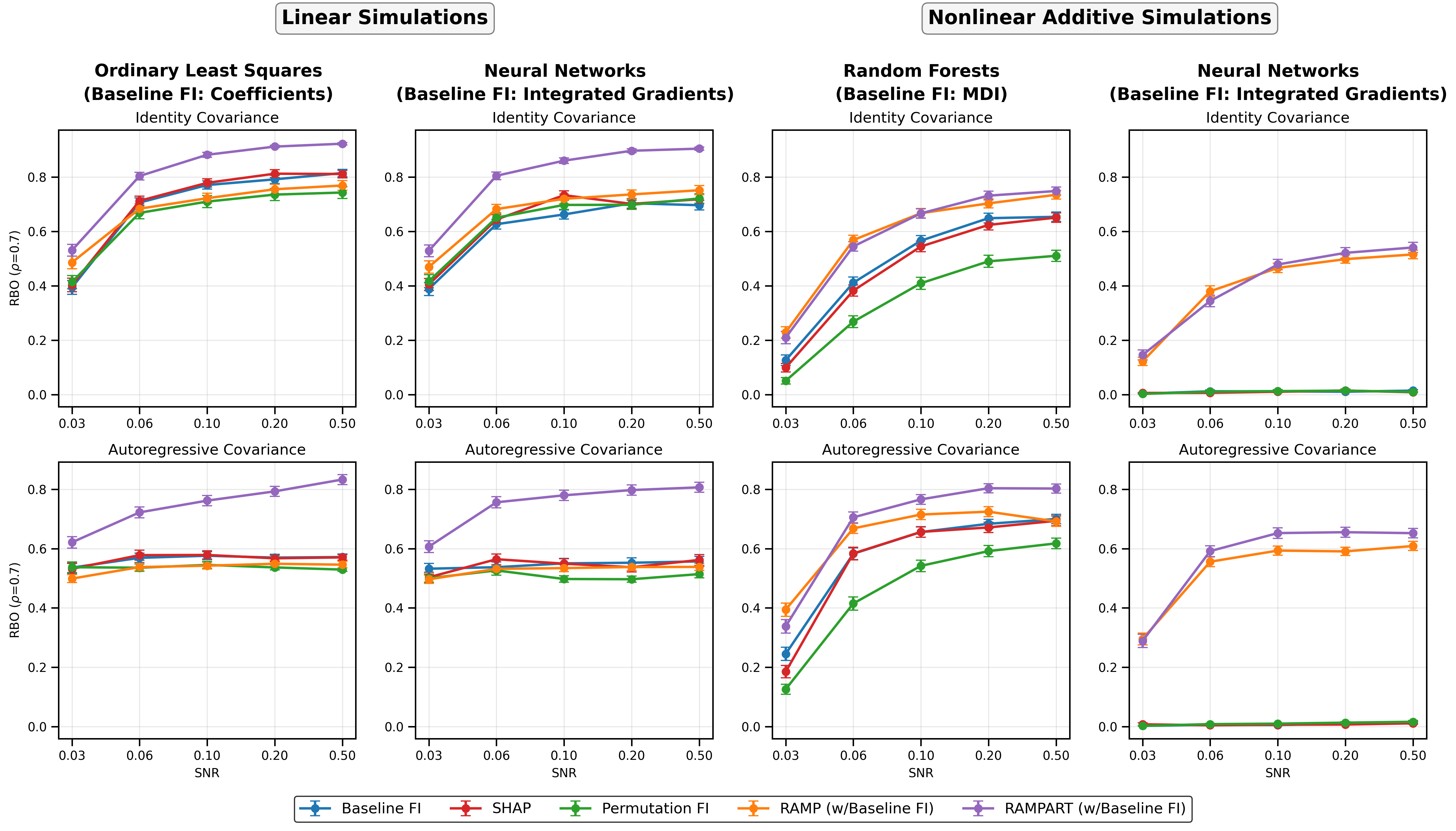}
\caption{Ranking accuracy (RBO with $\rho = 0.7$) for regression tasks  ($M = 1000$)} \label{fig: regig1000}
\end{center}
\end{figure}

\begin{figure}[hbt!]
\begin{center}
\includegraphics[width= \textwidth]{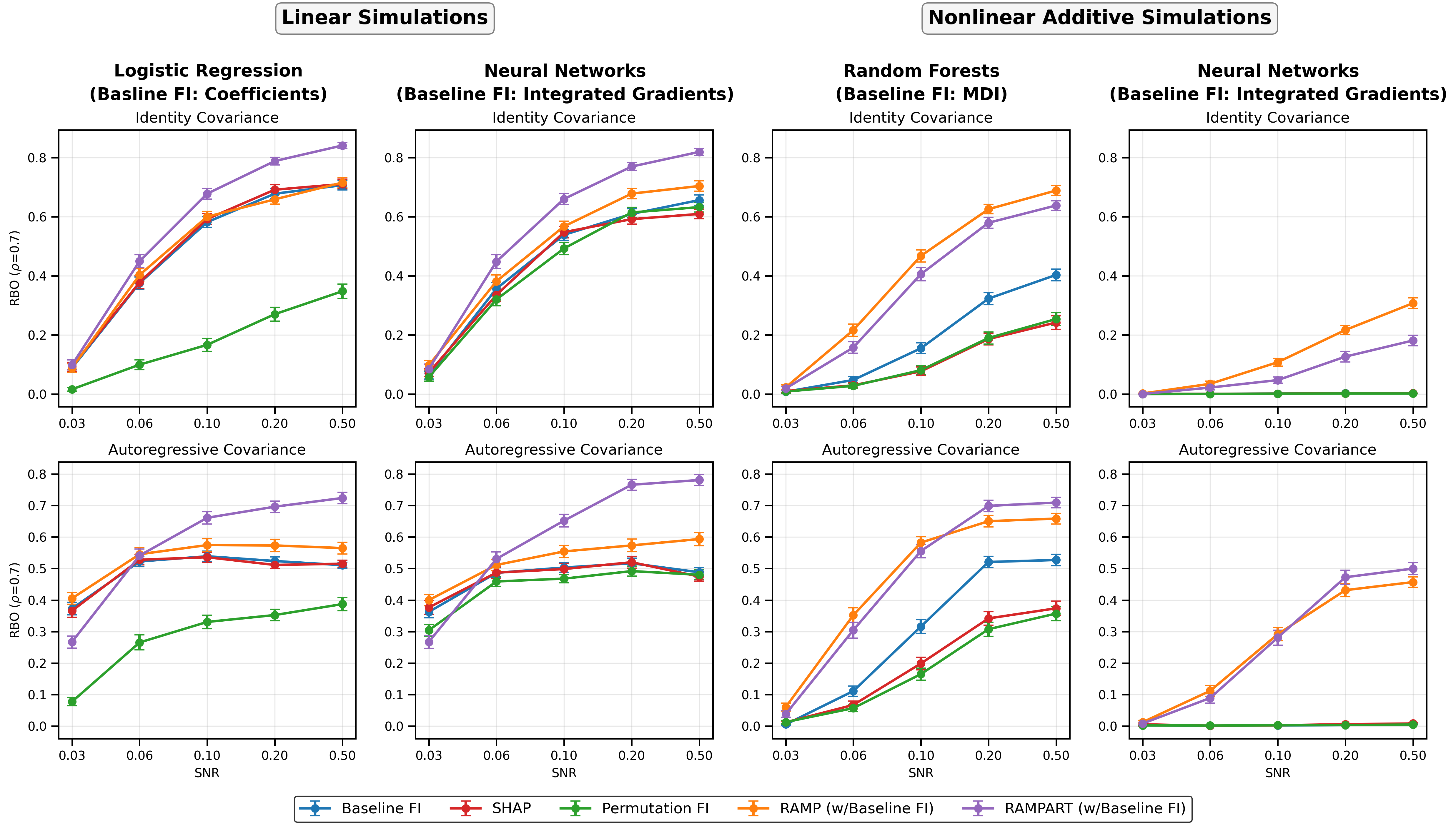}
\caption{Ranking accuracy (RBO with $\rho = 0.7$) for classification tasks  ($M = 2000$)} \label{fig: claig2000}
\end{center}
\end{figure}

\begin{figure}[hbt!]
\begin{center}
\includegraphics[width= \textwidth]{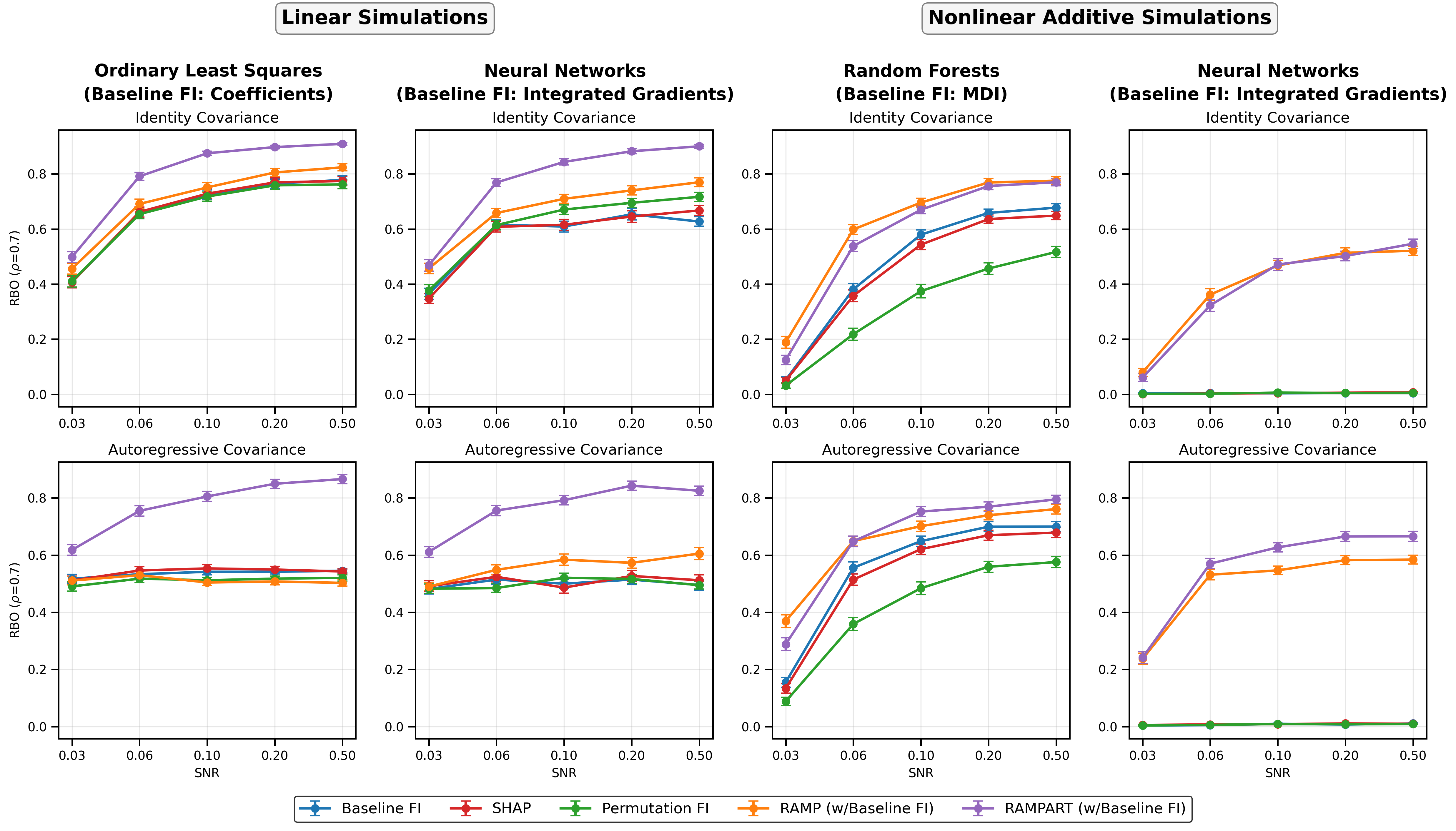}
\caption{Ranking accuracy (RBO with $\rho = 0.7$) for regression tasks  ($M = 2000$)} \label{fig: regig2000}
\end{center}
\end{figure}

\begin{figure}[hbt!]
\begin{center}
\includegraphics[width= \textwidth]{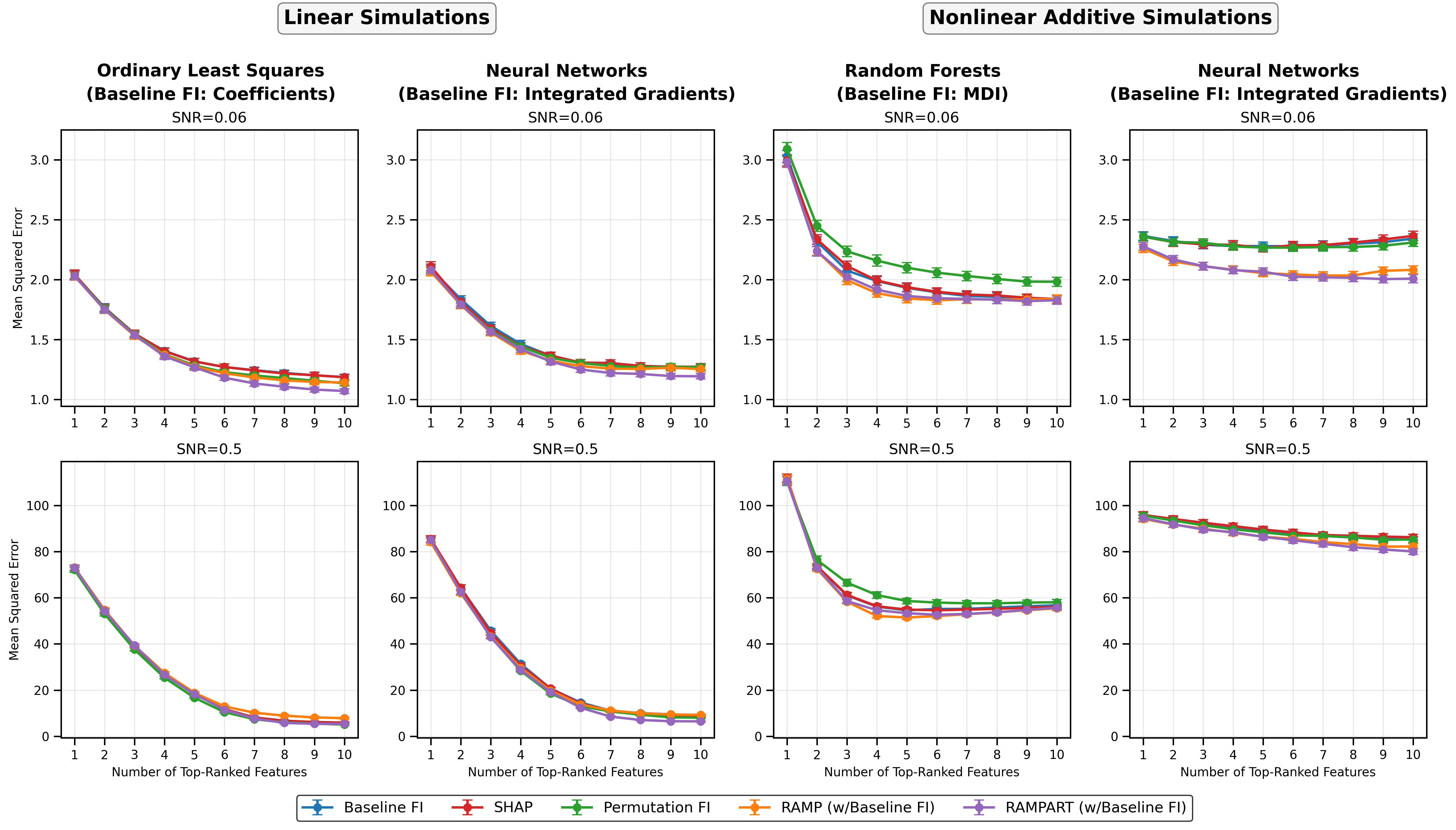}
\caption{Mean-Squared error vs. number of top-ranked features used as predictors in ablation simulation with identity covariance ($M = 500$)} \label{fig: ablation Reg_500}
\end{center}
\end{figure}

\begin{figure}[hbt!]
\begin{center}
\includegraphics[width= \textwidth]{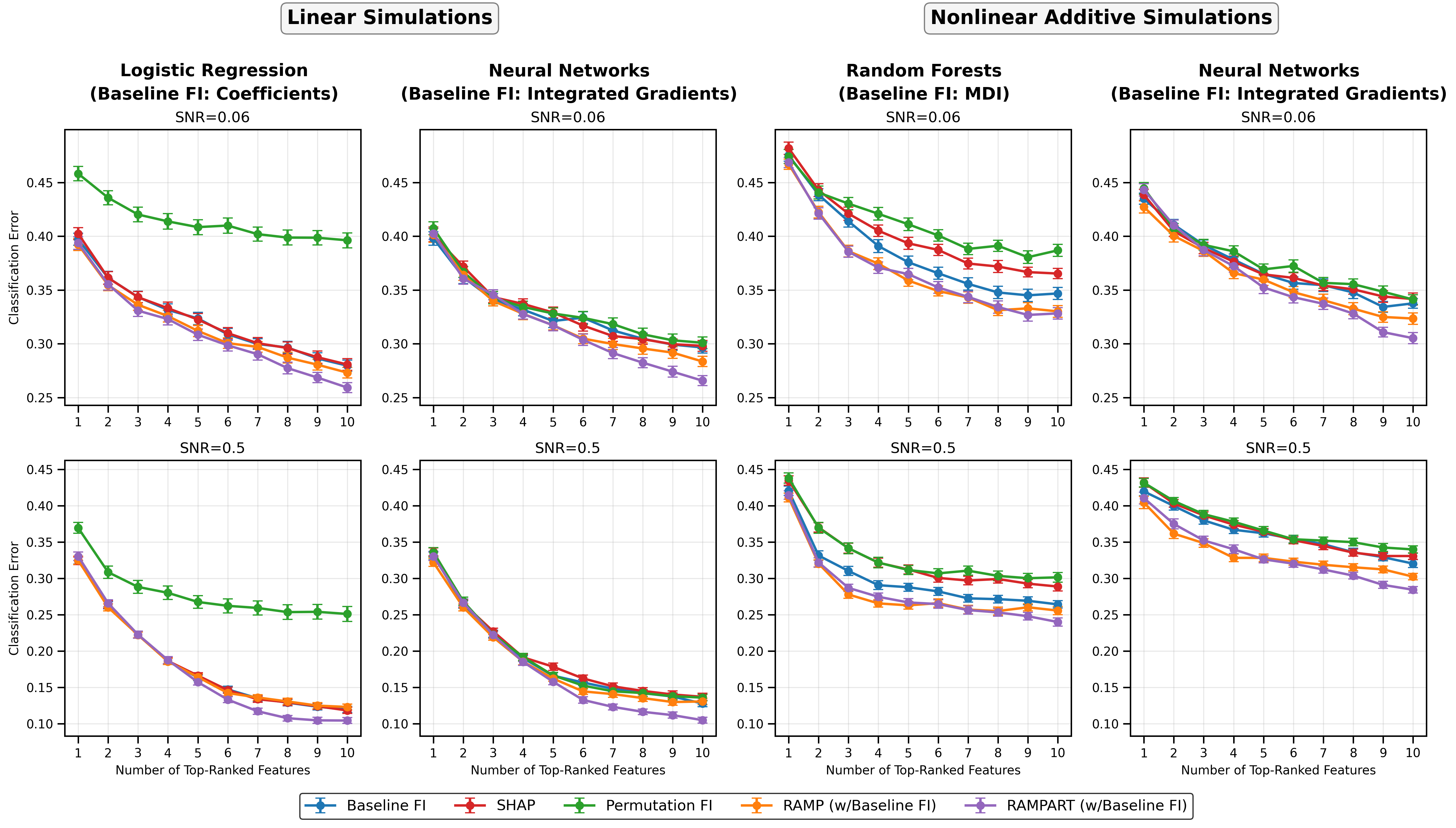}
\caption{Classification error vs. number of top-ranked features used as predictors in ablation simulation with identity covariance ($M = 1000$)} \label{fig: ablation Cla_1000}
\end{center}
\end{figure}

\begin{figure}[hbt!]
\begin{center}
\includegraphics[width= \textwidth]{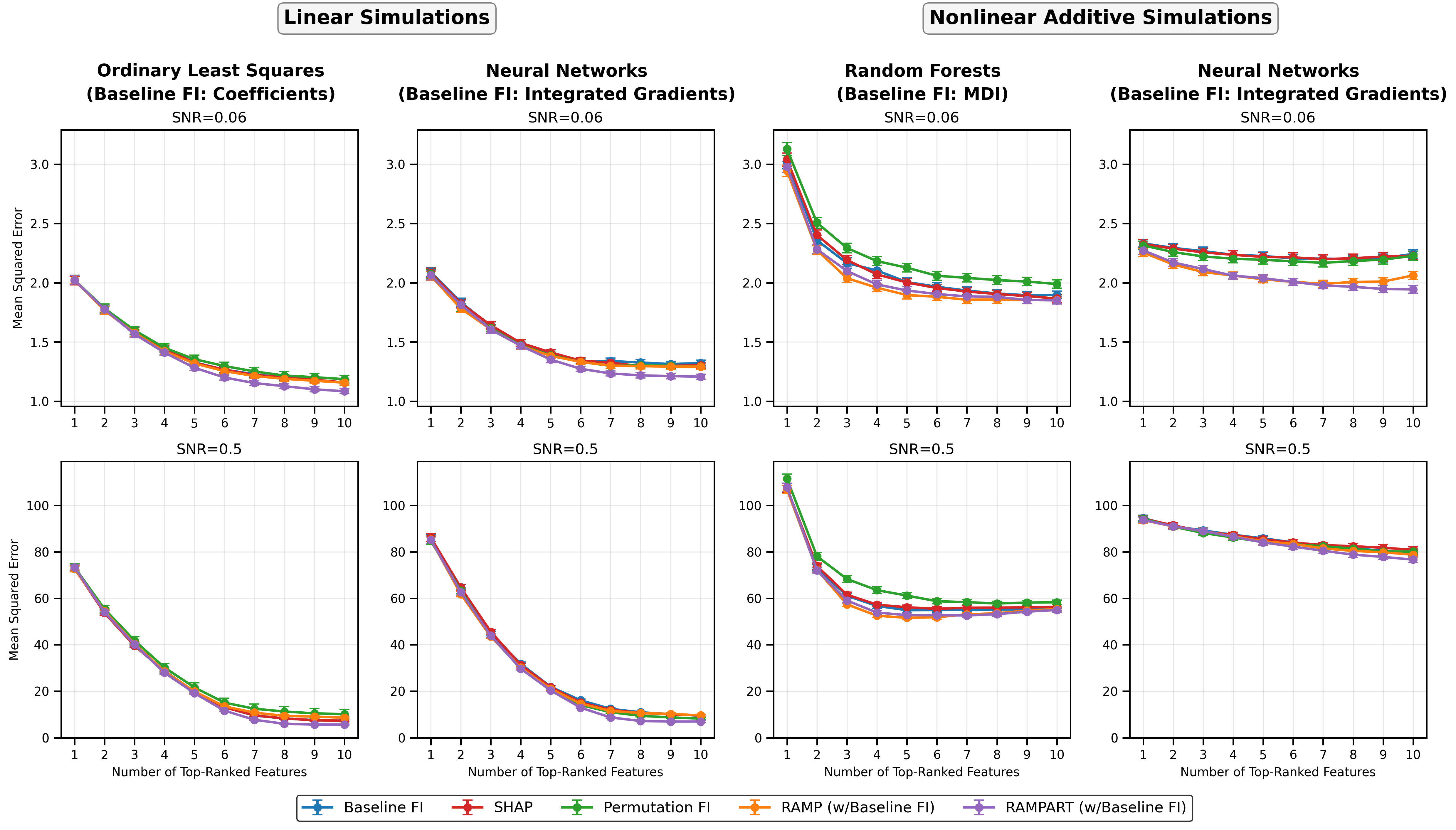}
\caption{Mean-Squared error vs. number of top-ranked features used as predictors in ablation simulation with identity covariance($M = 1000$)} \label{fig: ablation Reg_1000}
\end{center}
\end{figure}

\begin{figure}[hbt!]
\begin{center}
\includegraphics[width= \textwidth]{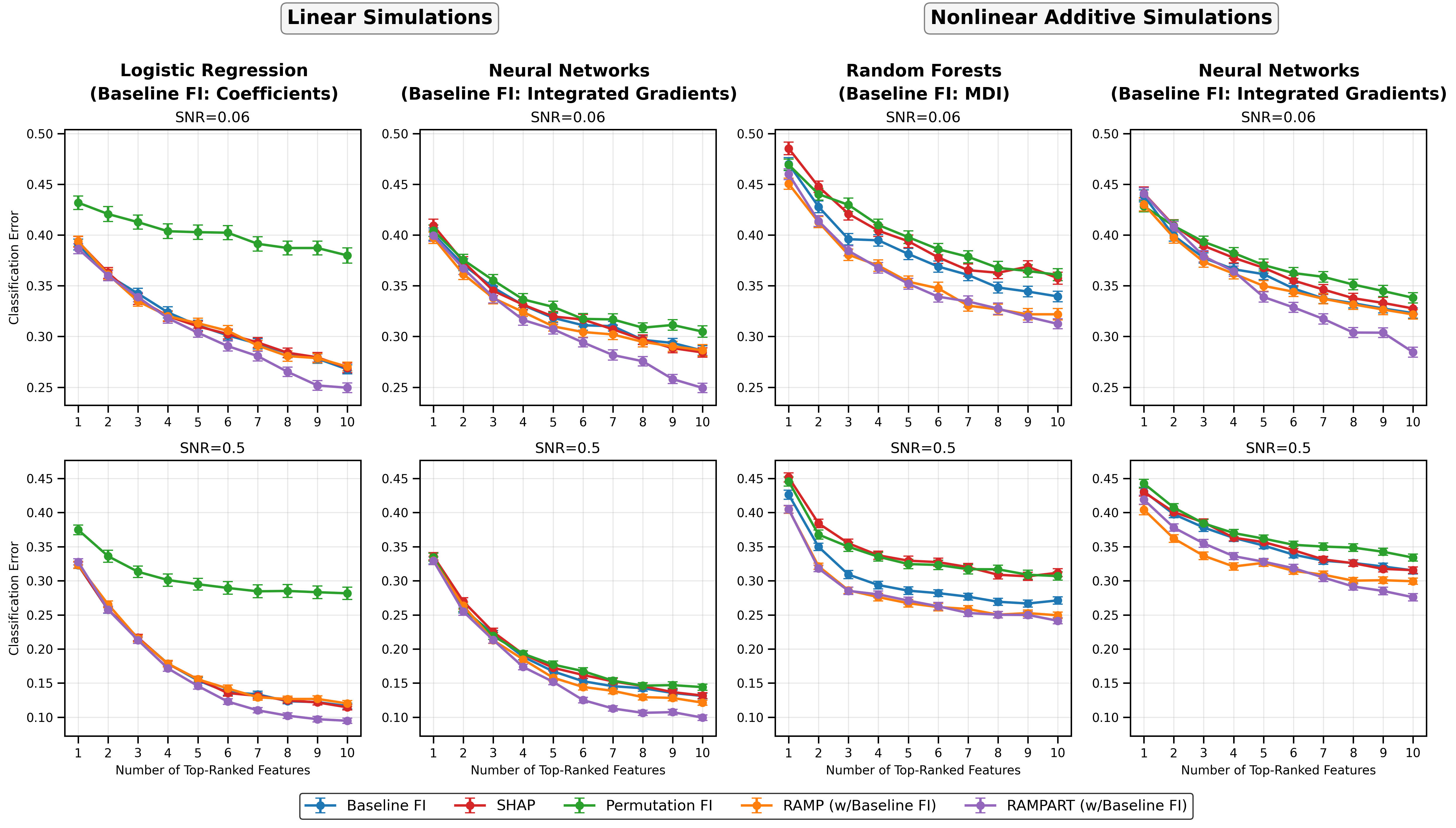}
\caption{Classification error vs. number of top-ranked features used as predictors in ablation simulation with identity covariance ($M = 2000$)} \label{fig: ablation Cla_2000}
\end{center}
\end{figure}

\begin{figure}[hbt!]
\begin{center}
\includegraphics[width= \textwidth]{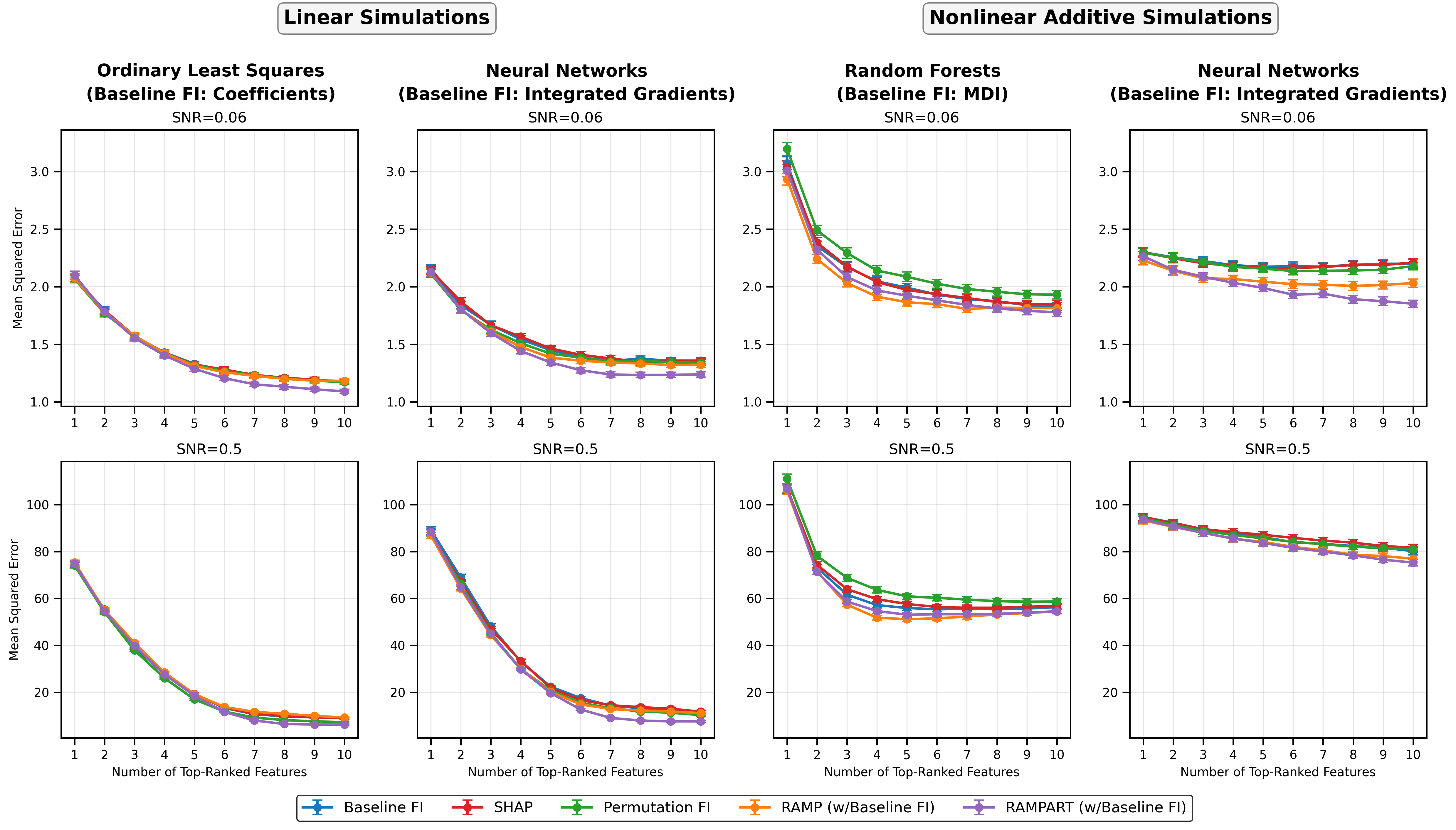}
\caption{Mean-Squared error vs. number of top-ranked features used as predictors in ablation simulation with identity covariance ($M = 2000$)} \label{fig: ablation Reg_2000}
\end{center}
\end{figure}

\begin{figure}[hbt!]
\begin{center}
\includegraphics[width= \textwidth]{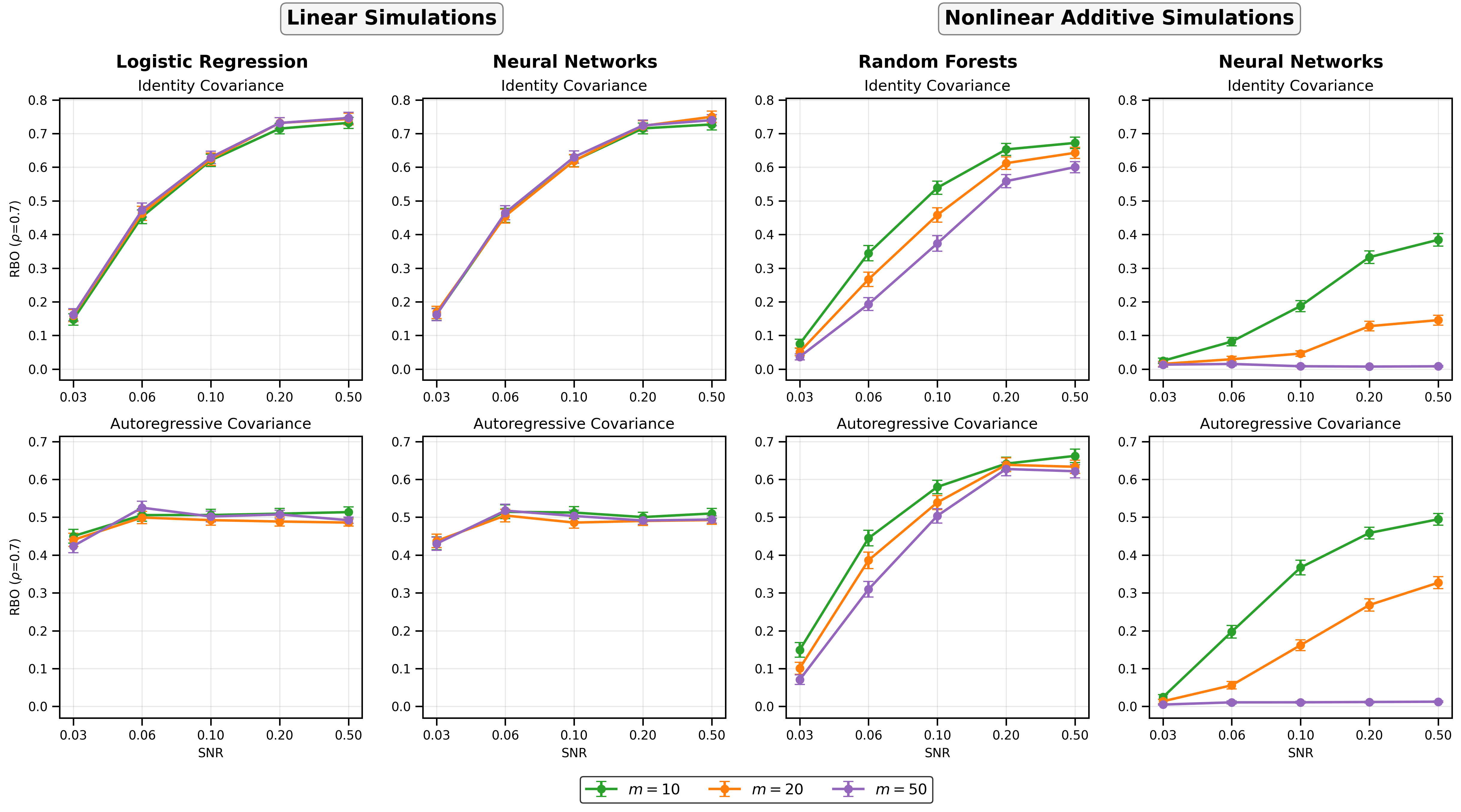}
\caption{Impact (RBO with $\rho = 0.7$) of minipatch size for classification ($M = 500$)} \label{fig: MP_Cla}
\end{center}
\end{figure}

\begin{figure}[hbt!]
\begin{center}
\includegraphics[width= \textwidth]{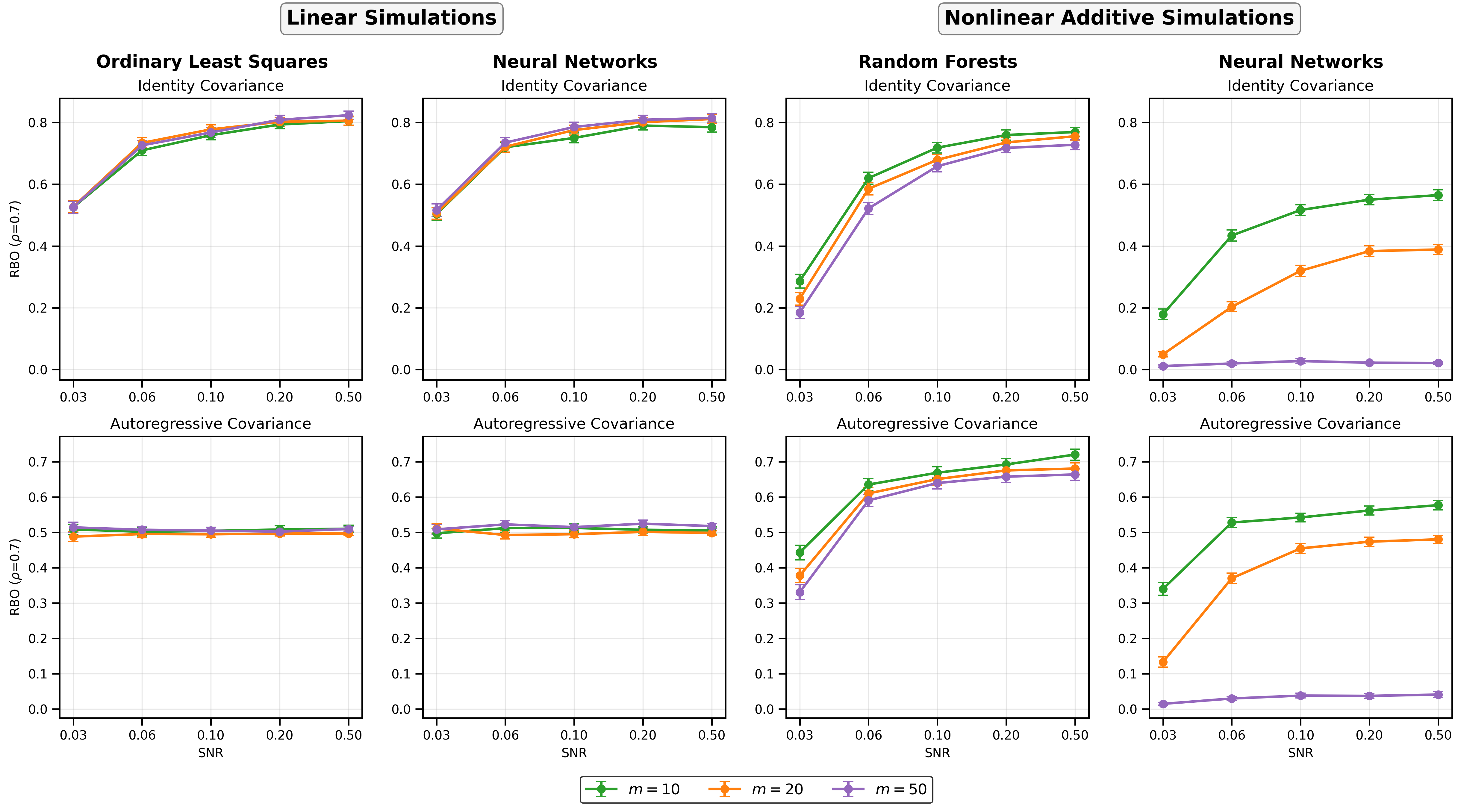}
\caption{Impact (RBO with $\rho = 0.7$) of minipatch size for regression ($M = 500$)} \label{fig: MP_Reg}
\end{center}
\end{figure}

\FloatBarrier

\section{Additional Case Study Discussion}
\label{appendix: CCLE}

\paragraph{Response Variable.} To quantify the PD-0325901 drug response in our case study, we use the area under the dose-response curve (AUC) as the primary outcome of interest. The AUC is a widely-used measure of overall drug sensitivity, defined as the area between the dose-response curve and $0$ \citep{barretina2012cancer}. A higher value indicates that the drug was more effective at killing the cancer cells. We refer to \citet{barretina2012cancer} for details on how this dose-response data was collected and processed.

\paragraph{Data Preprocessing of Gene Expression Data.} The raw CCLE data used in this case study can be downloaded from the DepMap Portal (\url{https://depmap.org/portal/download/}) (version 18Q3). Due to the heavy right-skewed distribution of the RNASeq gene expression values, we log-transformed ($log(x + 1)$) the raw gene expression data. We also restricted our analysis to the 1200 genes with the highest variance across the cell lines and sequentially removed genes which had $> 0.95$ Pearson correlation with another gene in the dataset, resulting in 1104 genes. The processed gene expression data was finally standardized to have zero mean and unit variance.

\begin{table}[h!]
\small
\begin{center}
\renewcommand{\arraystretch}{1} 
\begin{tabular}{cccccccc}
\toprule
\textbf{Gene} & \textbf{RAMPART} & \textbf{RAMP} & \textbf{Baseline} & \textbf{SHAP} & \textbf{Permutation} \\
\midrule
TOR4A & 1 & - & 4 & 2 & 1 \\
\cellcolor{gray!10}{ETV4} & \cellcolor{gray!10}{2} & \cellcolor{gray!10}{2} & \cellcolor{gray!10}{5} & \cellcolor{gray!10}{1} & \cellcolor{gray!10}{9} \\
SPRY2 & 3 & 1 & 7 & 5 & 2 \\
\cellcolor{gray!10}{GJB1} & \cellcolor{gray!10}{4} & \cellcolor{gray!10}{5} & \cellcolor{gray!10}{1} & \cellcolor{gray!10}{3} & \cellcolor{gray!10}{7} & \\
PYCARD & 5 & - & - & - & - \\
\cellcolor{gray!10}{WNT5A} & \cellcolor{gray!10}{6} & \cellcolor{gray!10}{-} & \cellcolor{gray!10}{-} & \cellcolor{gray!10}{-} & \cellcolor{gray!10}{-} \\
FERMT1 & 7 & - & - & - & 5 \\
\cellcolor{gray!10}{NRROS} & \cellcolor{gray!10}{8} & \cellcolor{gray!10}{-} & \cellcolor{gray!10}{-} & \cellcolor{gray!10}{-} & \cellcolor{gray!10}{-} \\
LYZ & 9 & 4 & 2 & 8 & 10 \\
\cellcolor{gray!10}{NPAS2} & \cellcolor{gray!10}{10} & \cellcolor{gray!10}{3} & \cellcolor{gray!10}{6} & \cellcolor{gray!10}{7} & \cellcolor{gray!10}{6} \\
RP11-290L1.3 & - & 6 & 9 & 10 & - \\
\cellcolor{gray!10}{ITGA6} & \cellcolor{gray!10}{-} & \cellcolor{gray!10}{7} & \cellcolor{gray!10}{10} & \cellcolor{gray!10}{6} & \cellcolor{gray!10}{3} \\
ID3 & - & 8 & 3 & 4 & - \\
\cellcolor{gray!10}{DUSP6} & \cellcolor{gray!10}{-} & \cellcolor{gray!10}{9} & \cellcolor{gray!10}{8} & \cellcolor{gray!10}{9} & \cellcolor{gray!10}{-} \\
TNFRSF14 & - & 10 & - & - & - \\
\cellcolor{gray!10}{TMEM184A} & \cellcolor{gray!10}{-} & \cellcolor{gray!10}{-} & \cellcolor{gray!10}{-} & \cellcolor{gray!10}{-} & \cellcolor{gray!10}{4} \\
RP11-284F21.10 & - & - & - & - & 8 \\
\bottomrule
\end{tabular}
\caption{Top-10 ranked genes according to each feature importance ranking method for predicting the PD-0325901 drug response.}
\label{table: top 10 genes}
\end{center}
\end{table}

\end{document}